\documentclass[10.5pt,letterpaper]{article}
\usepackage{fullpage}

\usepackage{amsmath,amssymb,amsthm}

\usepackage{color,mathtools}
\usepackage[round]{natbib}

\usepackage{multirow}
\usepackage{enumitem}

\usepackage{algorithm}
\usepackage{algorithmic}
\let\cite\citep

\usepackage{bm}
\newcommand{\bxi}{\boldsymbol{\xi}}

\usepackage{amssymb,amsmath,dsfont}

\providecommand{\lin}[1]{\ensuremath{\left\langle #1 \right\rangle}}
\providecommand{\abs}[1]{\left\lvert#1\right\rvert}
\providecommand{\norm}[1]{\left\lVert#1\right\rVert}

  \providecommand{\R}{\mathbb{R}} %
  \DeclareMathOperator{\E}{{\mathbb E}}
  \providecommand{\Eb}[1]{\E \left[#1\right] }             %
  \providecommand{\EE}[2]{\E_{#1} \! #2 }      %
  \providecommand{\EEb}[2]{\E_{#1}\!\! \left[#2\right] } %

  \DeclareMathOperator*{\argmin}{arg\,min}

  \providecommand{\0}{\mathbf{0}}
  
  \renewcommand{\aa}{\mathbf{a}}
  \providecommand{\bb}{\mathbf{b}}

  \providecommand{\ee}{\mathbf{e}}

  \renewcommand{\gg}{\mathbf{g}}

  \providecommand{\vv}{\mathbf{v}}
  
  \providecommand{\xx}{\mathbf{x}}
  \providecommand{\yy}{\mathbf{y}}

  \providecommand{\cC}{\mathcal{C}}
  \providecommand{\cD}{\mathcal{D}}
  
  \providecommand{\cF}{\mathcal{F}}

  \providecommand{\cO}{\mathcal{O}}

  \providecommand{\cS}{\mathcal{S}}

\newtheorem{lemma}{Lemma}
\newtheorem{definition}[lemma]{Definition}
\newtheorem{remark}[lemma]{Remark}
\newtheorem{assumption}{Assumption}
\newtheorem{theorem}[lemma]{Theorem}
\newtheorem{example}{Example}  

\RequirePackage[colorinlistoftodos,bordercolor=orange,backgroundcolor=orange!20,linecolor=orange,textsize=scriptsize]{todonotes}
\providecommand{\comment}[2]{\todo[inline,caption={}]{\textbf{#1: }#2}}%
\providecommand{\inlinecomment}[3]{%
  {\color{#1}#2: #3}}%
\newcommand\commenter[2]%
{%
  \expandafter\newcommand\csname i#1\endcsname[1]{\inlinecomment{#2}{#1}{##1}}
  \expandafter\newcommand\csname #1\endcsname[1]{\comment{#1}{##1}}
}

\definecolor{mydarkblue}{rgb}{0,0.08,0.45}
\usepackage[colorlinks=true,linkcolor=blue]{hyperref} 
\hypersetup{ %
    colorlinks=true,
    linkcolor=mydarkblue,
    citecolor=mydarkblue,
    filecolor=mydarkblue,
    }
\usepackage[capitalize,noabbrev]{cleveref}

\usepackage{url}

\usepackage{mathtools}
\usepackage{enumitem}
\newtheorem{taggedassumptionx}{Assumption}
\newenvironment{taggedassumption}[1]
 {\renewcommand\thetaggedassumptionx{#1}\taggedassumptionx}
 {\endtaggedassumptionx}

\let\cite\citep

\commenter{seb}{red}
\commenter{pr}{blue}

\begin{document}

\title{The Error-Feedback Framework: Better Rates for SGD\\ with Delayed Gradients and Compressed Updates}%

\author{Sebastian U. Stich\thanks{\texttt{\{sebastian.stich,sai.karimireddy\}@epfl.ch}, Machine Learning and Optimization Lab (MLO), EPFL, Switzerland.} \\ EPFL \and Sai Praneeth Karimireddy\footnotemark[1] \\ EPFL%
}

\date{}

\maketitle

\begin{abstract}%
We analyze (stochastic) gradient descent (SGD) with delayed updates on smooth quasi-convex and non-convex functions and derive concise, non-asymptotic, convergence rates. 
We show that the rate of convergence in all cases consists of two terms: (i) a stochastic term which is not affected by the delay, and (ii) a higher order deterministic term which is only linearly slowed down by the delay. Thus, in the presence of noise, the effects of the delay become negligible after a few iterations and the algorithm converges at the same optimal rate as standard SGD. This result extends a line of research that showed similar results in the asymptotic regime or for strongly-convex quadratic functions only. 

We further show similar results for SGD with more intricate form of delayed gradients---compressed gradients under error compensation and for local~SGD where multiple workers perform local steps before communicating with each other. In all of these settings, we improve upon the best known rates.

These results show that SGD is robust to compressed and/or delayed stochastic gradient updates. This is in particular important for distributed parallel implementations, where asynchronous and communication efficient methods are the key to achieve linear speedups for optimization with multiple devices.
\end{abstract}

\section{Introduction}
\label{sec:intro}
We consider the unconstrained optimization problem
\begin{align}
 f^\star := \min_{\xx \in \R^d} f(\xx)\,, \label{eq:problem}
\end{align}
for a quasi-convex (i.e.\ 1-quasar convex) or non-convex smooth function $f \colon \R^d \to \R$ and study a variety of stochastic gradient methods with delayed (or \emph{stale}) updates. Stochastic gradient descent (SGD) methods~\cite{Robbins:1951sgd} generate a sequence $\{\xx_t\}_{t \geq 0}$ of iterates for an arbitrary starting point $\xx_0 \in \R^d$ and positive stepsizes $\{\gamma_t\}_{t \geq 0}$, by sequential updates of the form
\begin{align}
 \xx_{t+1} &= \xx_t - \gamma_t \gg_t\,,
 & &\text{where}
 & \gg_t &= \nabla f(\xx_t) + \bxi_t\,, \tag{SGD} \label{eq:syncSGD}
\end{align}
is a \emph{stochastic gradient} for zero-mean noise terms $\{\bxi_t\}_{t \geq 0}$. When the noise is zero almost surely, then we recover the classic \emph{gradient descent} method as a special case. SGD is the state of the art optimization method for many machine learning---especially deep learning---optimization problems~\cite{Bottou2010:sgd}. In order to use the compute power of many parallel devices, it is essential to depart from the inherently serial updates as in~\eqref{eq:syncSGD}. For instance, in mini-batch SGD~\cite{Dekel2012:minibatch} several stochastic gradients are computed at the same iterate $\xx_t$ (an operation which can be parallelized, but still requires synchronization among the devices). Fully asynchronous methods, where the devices operate completely independently, and e.g.\ write their updates to a shared memory~\cite{Niu2011:hogwild} perform often better in practice, as the effect of stragglers (slow devices) is minimized. In an orthogonal line of work, gradient compression techniques have been developed with the aim to reduce the communication overhead between the devices~\cite{Alistarh2017:qsgd}. We analyze methods of both these types in this paper.

Stochastic gradient descent on $\mu$-strongly convex functions has asymptotically the iteration complexity $\smash{\cO\bigl(\frac{\sigma^2}{\mu \epsilon}\bigr)}$ for sufficiently small $\epsilon \to 0$~\cite{Polyak1990:averaging,Nemirovski2009:sgd} and where here $\sigma^2$ is an upper bound on the noise, $\E \norm{\bxi_t}^2 \leq \sigma^2$, $\forall t \geq 0$ (we discuss more general bounds below). \citet{Chaturapruek2015:noise} show that under certain regularity conditions, asynchronous SGD reaches the same asymptotic convergence rate as the standard (serial) SGD. 
For gradient compression techniques with \emph{error compensation}---a technique first described in e.g.\ \cite{Seide2015:1bit,Strom2015:1bit}---\citet{Stich2018:sparsified} show that the asymptotic convergence rate $\smash{\cO\bigl(\frac{G^2}{\mu \epsilon}\bigr)}$, where $G^2$ denotes an upper bound on the second moment of the stochastic gradients, is attained for host of compression operators, such as e.g. sparsification, quantization or (biased) greedy selection.
 These two results show that asynchronous methods and gradient compression can both be used to hide communication overheads---asymptotically---\emph{for free}. Thus they are very interesting techniques for distributed optimization. In this work we aim to derive tight (non-asymptotic) convergence rates to deepen our understanding of these schemes.

The starting point for our analysis is the recent work of~\cite{Arjevani2018:delayed} that studies the \emph{delayed SGD} \eqref{eq:dSGD} algorithm for a fixed (integer) delay $\tau \geq 1$, given as
\begin{align}
 \xx_{t+1} = \xx_t - \gamma_{t-\tau} \gg_{t-\tau} \,, \tag{D-SGD} \label{eq:dSGD}
\end{align}
for $t \geq \tau$, and $\xx_0 = \xx_1 = \dots = \xx_{\tau}$ for the first iterations. Here $\gg_{t-\tau} = \nabla f(\xx_{t-\tau}) + \bxi_{t-\tau}$ is a stochastic gradient computed at $\xx_{t-\tau}$, instead of at $\xx_{t}$ as in the vanilla scheme~\eqref{eq:syncSGD}. \citet{Arjevani2018:delayed} analyze~\eqref{eq:dSGD} on convex quadratic functions. For $\mu$-strongly convex, $L$-smooth quadratic function with minimum at $\xx^\star$, they show that the suboptimality gap decreases as $\smash{\tilde  \cO\bigl(L \norm{\xx_0-\xx^\star}^2 \exp\bigl[ -\frac{\mu T}{10 L \tau}\bigr] + \frac{\sigma^2}{\mu T}  \bigr)}$ after $T$ iterations, i.e.\ the algorithm achieves iteration complexity $\smash{\tilde \cO \bigl(\frac{\sigma^2}{\mu \epsilon} + \frac{L\tau}{\mu} \log \frac{1}{\epsilon} \bigr)}$.\footnote{Following standard convention, the $\cO$-notation hides constant factors, and the $\tilde \cO$-notation hides constants and factors polylogarithmic  in the problem parameters.} Again, we see that asymptotically, when $T \to \infty$ or $\epsilon \to 0$, the effect of the delay $\tau$ is negligible when $\sigma^2 > 0$. The delay $\tau$ only appears in the so-called optimization term that is only dominant for small $\sigma^2$ (and especially for deterministic delayed gradient descent where $\sigma^2=0$). The linear dependency on $\tau$ is optimal and cannot further be improved.
These results were obtained with a technique based on generating functions---an approach that seems limited to quadratic functions. In this work we use the error-feedback framework to extend their results to general convex, and non-convex functions. Further, we also analyze more intricate forms of delays in the gradients---compressed gradients with error compensation, and local~SGD.

\subsection{Main Contributions and Structure}
Our main contributions are:
\begin{itemize}
 \item In Section~\ref{sec:delayed} we generalize the analysis of~\cite{Arjevani2018:delayed} of~\eqref{eq:dSGD} to quasi-convex functions, and show the iteration complexity $\tilde \cO \bigl(\frac{\sigma^2}{\mu \epsilon} + \frac{L\tau}{\mu} \log \frac{1}{\epsilon} \bigr)$ for strongly ($\mu > 0$) and $\smash{\cO \bigl(\frac{L\tau \norm{\xx_0 - \xx^\star}^2}{\epsilon} + \frac{\sigma^2 \norm{\xx_0 - \xx^\star}^2 }{\epsilon^2} \bigr)}$ for general ($\mu=0$) quasi convex functions. The dependency on the problem parameters $\tau$, $L$, $\mu$ is in tight up to logarithmic factors (for \emph{accelerated} delayed gradient methods---which we do not consider here---these rates could be improved). Further, for arbitrary smooth non-convex functions, we show a iteration complexity of $\cO\bigl( \smash{ \frac{L\tau (f(\xx_0) - f^\star) }{ \epsilon } + \frac{L\sigma^2(f(\xx_0) - f^\star)}{\epsilon^2}}  \bigr)$ for convergence to a stationary point i.e convergence of the squared gradient norm to zero.
 \item In Section~\ref{sec:efsgd} we generalize the analysis of~\cite{Stich2018:sparsified} for SGD with gradient compression and error compensation to quasi-convex functions and show the iteration complexity $\smash{\tilde \cO \bigl(\frac{\sigma^2}{\mu \epsilon} + \frac{L}{\mu \delta} \log \frac{1}{\epsilon} \bigr)}$ for strongly ($\mu > 0$) and $\cO \bigl(\frac{L \norm{\xx_0 - \xx^\star}^2}{\delta \epsilon} + \frac{\norm{\xx_0 - \xx^\star}^2 \sigma^2}{\epsilon^2} \bigr)$ for general ($\mu=0$) quasi convex
 functions. Here $\delta > 0$ is a parameter that measures the compression quality. For general smooth non-convex functions, we show an iteration complexity of $\cO\bigl(\frac{L (f(\xx_0) - f^\star) }{ \delta \epsilon } + \frac{L\sigma^2(f(\xx_0) - f^\star)}{\epsilon^2}  \bigr)$ for convergence to a stationary point. This is the first analysis of these methods without the bounded gradient assumption and improves over all previous results. 
 In particular, all previous results suffer from a quadratic dependence on $\delta$ whereas our rates have only a linear dependence.
 \item In Section~\ref{sec:localsgd} we derive complexity estimates for local SGD. This algorithm can be viewed as a special asynchronous stochastic gradient method and has become increasingly popular in recent years~\cite{Zinkevich2010:parallelSGD,
 Zhang2016:averaging,Lin2018:local,
 Dieuleveut2019:local}. Our complexity estimate improve previous results, but we do not believe that our bounds are tight. 
\end{itemize}
We discuss the precise setting in Section~\ref{sec:setting} and highlight important cases. We further provide some key technical lemmas in Section~\ref{sec:technical}.

\subsection{Related Work}
For in-depth discussion of SGD and its application in machine learning we refer to the book of~\cite{Bottou2018:book}. Here we try to list the most closely related work by topic.

\paragraph{Asynchronous and delayed SGD.} 
Asynchronous methods have  been  intensively studied over the last three decades, starting with~\cite{Bertsekas1989:parallel}. A large impact for distributed machine learning had the \textsc{Hogwild!} algorithm~\cite{Niu2011:hogwild} that performs asynchronous (block)-coordinate updates on a shared parameter vector. Many early theoretical results for asynchronous methods depend on rigorous sparsity assumptions~\cite{Chaturapruek2015:noise,Mania2017:perturbed,Leblond2018:improved}.

An algorithm very similar to~\eqref{eq:dSGD}, but with arbitrary (instead of fixed) delays of at most $\tau$ iterations, was studied in~\cite{Agarwal2011:delayed}. For smooth convex functions they show a bound of $\smash{\cO \bigl(\frac{\sigma}{\sqrt{T}} + \frac{\tau^2}{\sigma^2 T} \bigr)}$ (in terms of $\sigma, \tau, T$ only, $\tau \geq 1$). \citet{Feyzmahdavian2016:async} improve the bound to $\smash{\cO \bigl(\frac{\sigma}{\sqrt{T}} +  \frac{\tau^2}{T}  \bigr)}$. 
\citet{Arjevani2018:delayed} show that for~\eqref{eq:dSGD} the dependency on $\tau$ can be improved and show a bound $\smash{\cO \bigl(\frac{\sigma}{\sqrt{T}} +  \frac{\tau}{T}  \bigr)}$ on convex quadratic functions. Our work extends these results to a boarder class of functions. Our proof technique is different from theirs and also allows the analysis of more general delay models, for e.g.\ the variables delays as in~\cite{Agarwal2011:delayed,sra16:adadelay}. For general non-convex smooth functions, \citet{lian2015asynchronous} show that after $\Omega \bigl(\frac{L \tau^2 (f(\xx_0) - f^\star)}{\sigma^2} \bigr)$ iterations the effect of the delay becomes negligible and we recover their result with our analysis.

\paragraph{Compressed gradient methods and error compensation.}
Convergence aspects of (stochastic) gradient descent with compressed gradients have been studied in various communities, and results for unbiased compression (perturbations) can be traced back to e.g.~\cite{Polyak1987:book}. Jointly with the increase of the size of the deep learning models, the interest in gradient compression techniques has risen in the past years~\cite{Wen2017:terngrad,Alistarh2017:qsgd,Wangni2018:sparsification}. 
The convergence analysis of these methods, see e.g.~\cite{Alistarh2017:qsgd}, typically give rates of the form $\smash{\cO \bigl(\frac{\omega \sigma^2}{\mu T}\bigr)}$ where $\omega \geq 1$ is a parameter that measures the additional noise introduced by the (unbiased) compression operators. Despite the practical success of these methods, the linear slowdown in $\omega$ makes them less attractive from a theoretical point of view.

A different type of methods use error-correction, or other error-compensation mechanisms. A method of this type was for instance developed for a particular application in~\cite{Seide2015:1bit,Strom2015:1bit}. \citet{Wu:2018error} analyze a method with error correction for quadratic functions, \citet{Stich2018:sparsified} provide an analysis for strongly convex functions and a large class of compression operators, including biased compressors. As a key result they show that the optimal $\smash{\cO \bigl(\frac{\sigma^2}{\mu T}\bigr)}$ convergence rate can be attained, with the same asymptotic rate as the schemes without error compensation. These results were extended in~\cite{Karimireddy2019:error} to non-smooth and non-convex functions. Here we analyze this method in a more general setting, for instance without the bounded gradient assumption.

\paragraph{Local SGD.}
Local SGD (a.k.a.\ parallel SGD) is parallel version of SGD, where each device performs local updates of the form~\eqref{eq:syncSGD} in parallel on the local data, and the devices average their iterates after every~$\tau$ updates. This is different from mini-batch SGD where the averaging happens after every iteration, but more closely related to mini-batch SGD with $\tau$-times larger batchsizes on each device. This algorithm has attracted the attention of the community due to its application in federated learning~\cite{Mcmahan2016:communication}. Early analyses focused on variants with only one averaging step~\citep{Mann2009:parallelSGD,Zinkevich2010:parallelSGD,Zhang2013:averaging,
Shamir2014:distributed,Godichon2017:oneshot,
Jain2018:parallel}. 
More practical are schemes that perform more frequent averaging of the parallel sequences~\cite{Zhang2016:averaging,Lin2018:local}. Analyses have been developed for strongly convex~\cite{Dieuleveut2019:local,Stich2018:local} and non-convex~\cite{Yu2018parallel,Wang2018cooperative} functions. These results show that local SGD can attain the optimal convergence rate of SGD when $T$ is large enough compared to $\tau$. For non-convex functions, the best bounds (with respect to only the parameter $\tau$) are $T=\Omega(\tau^4)$~\cite{Yu2018parallel} and for strongly convex functions a better quadratic dependence $T=\Omega(\tau^2)$ is known~\cite{Dieuleveut2019:local,Stich2018:local}. Here we improve this to $T=\tilde\Omega(\tau)$ which is optimal up to logarithmic factors (we need $T \geq \tau$ to communicate at least once). However, a closer inspection of our bounds reveals that they are not yet tight in many cases. They also do not match with the lower bounds~\cite{Arjevani2015:lowerbound,Woodworth2018:graph}.

\paragraph{Proof techniques.}
Our proof consists of three parts: firstly (i), we follow closely the analysis in~\cite{Stich2018:sparsified,Karimireddy2019:error} to derive a one-step progress estimate. This technique is based on ideas of the perturbed iterate analysis~\cite{Mania2017:perturbed,Leblond2018:improved} in combination with standard estimates~\cite{Nesterov2004:book}. Secondly (ii), to derive the final complexity estimates and getting the optimal optimization terms in the rate, we use the technique from~\cite{Stich2019:sgd} (there would be other options here, see for instance~\cite{Stich2020:compression}). Thirdly (iii), whilst we follow similar techniques to estimate the error as in previous works, we split the error term in a bias and noise component. This allows to use bigger stepsizes: whilst e.g.~\citet{Feyzmahdavian2016:async} had to use stepsizes $\cO\bigl(\frac{1}{\tau^2}\bigr)$, we can use stepsizes $\cO\bigl(\frac{1}{\tau}\bigr)$ (only showing the dependency on $\tau$), similar as in~\cite{Arjevani2018:delayed}.

\paragraph{Follow up advances.}
Since the initial submission, \citet{Karimireddy2019scaffold,Woodworth2019local} build upon our techniques to improve the results for local SGD. \citet{Woodworth2019local} use an improved step-size to show that local SGD can sometimes converge faster, even beating large batch SGD. They also construct a lower bound example proving that their analysis is tight. \citet{Karimireddy2019scaffold} show that using two step-sizes (a global and local step-size) can give faster rates. 
\citet{koloskova2020decentralized} analyze local SGD the setting where the data across the devices is heterogenous and 
\citet{Karimireddy2019scaffold} %
propose a new algorithm to overcome this heterogeneity.

\section{Formal Setting}
\label{sec:setting}
In this section we discuss our different settings and assumptions. For each of the problems studied in the later sections (delayed updates, compressed gradients, and local SGD) we analyze three cases: when $f$ is a (i) strongly quasi-convex, (ii) general quasi-convex, or a (iii) arbitrary smooth (i.e.\ not comprised in classes (i) or (ii)) non-convex function. 

We first examine the notion of quasi-convexity with respect to a minimizer $\xx^\star \in \R^d$ of~\eqref{eq:problem}. This is a substantial relaxation of the standard convexity assumption, as the assumption also holds for certain non-convex functions, such as e.g.\ star convex functions. The definition coincides with the recently introduced class of $(1,\mu)$-quasar convex functions~\cite{Hinder2019:quasar}. Our definition is slightly less general than the notion of quasi-convexity introduced in~\cite{Necoara2019:linear}, as we choose a particular minimizer~$\xx^\star$, the ``quasar-convex point'' and do not e.g.\ use projection on the set of minimizers as in~\cite{Necoara2019:linear}. However, extension of the analysis to these settings would be possible.

\begin{assumption}[$\mu$-quasi-convexity (w.r.t.\ $\xx^\star$)]
\label{ass:strong}
The function $f \colon \R^d \to \R$ is differentiable and \emph{$\mu$-quasi convex} for a constant $\mu \geq 0$ \emph{with respect to} 
$\xx^\star$, that is
\begin{align}
 f(\xx)-f^\star + \frac{\mu}{2}\norm{\xx-\xx^\star}^2 \leq \lin{\nabla f(\xx),\xx-\xx^\star}\,, \qquad \forall \xx \in \R^d.  \label{def:strong}
\end{align}
\end{assumption}
\begin{remark}
Note that $f$ can be quasi-convex w.r.t.\ $\xx^\star$ only if $\xx^\star \in \argmin_{\xx \in \R^d} f(\xx)$. If $\mu$ is possibly 0, we say the function is \emph{general} quasi-convex. When $\mu > 0$, this assumption implies that such a $\xx^\star$ is unique and we say the function is \emph{strongly} quasi-convex.
\end{remark}
\begin{remark} For $\mu$-strongly convex functions (under the standard definition), it holds
\begin{align*}
  f(\xx)-f(\yy) + \frac{\mu}{2}\norm{\xx-\yy}^2 \leq \lin{\nabla f(\xx),\xx-\yy}\,, \qquad \forall \xx,\yy \in \R^d.
\end{align*}
Thus, by setting $\yy= \xx^\star$ we see that~\eqref{def:strong} is more general. Interestingly, \eqref{def:strong} also holds for non-convex functions. We give a few examples in Section~\ref{sec:quasi-convex}.%
\end{remark}
\begin{remark}
For functions which satisfy the Polyak-\L{}ojasiewicz condition \cite{Karimi2016pl}, we have 
\[
\norm{\nabla f(\xx)}^2 \geq 2\mu(f(\xx) - f(\xx^\star))\,, \qquad \forall\,\xx \in \R^d\,.
\]
This is a weaker condition than quasi-convexity since~\eqref{def:strong} implies that
\[
    f(\xx)-f^\star + \frac{\mu}{2}\norm{\xx-\xx^\star}^2 \leq \lin{\nabla f(\xx),\xx-\xx^\star} \leq \frac{1}{2\mu}\norm{\nabla f(\xx)}^2 + \frac{\mu}{2}\norm{\xx-\xx^\star}^2\,.
\]
\end{remark}

\noindent We will further assume that the gradients of the function $f$ are Lipschitz, and hence that $f$ is smooth. 
\begin{assumption}[L-smoothness]
\label{ass:lsmooth}
The function $f \colon \R^d \to \R$ is differentiable and there exists a constant $L \geq 0$ such that
\begin{equation}\label{def:lgradlipschitz}
    \norm{\nabla f(\xx) - \nabla f(\yy)} \leq L\norm{\xx - \yy}\,, \qquad \forall \xx, \yy \in \R^d\,.
\end{equation}
\end{assumption}

\noindent We will next describe some implications of Assumption \ref{ass:lsmooth} which will be later useful.
\begin{remark}
The Lipschitz gradient condition \eqref{def:lgradlipschitz} implies that there exists a quadratic upper bound on $f$~\citep[Lemma 1.2.3]{Nesterov2004:book}:
\begin{align}
 f(\yy) \leq f(\xx) + \lin{\nabla f(\xx),\yy - \xx} + \frac{L}{2}\norm{\yy - \xx}^2\,, \qquad \forall \xx, \yy \in \R^d\,.
 \label{eq:quad-smooth}
\end{align}
Further, minimizing both the left and right hand side with respect to $\yy$ of \eqref{eq:quad-smooth} yields
\begin{align}
 \norm{ \nabla f(\xx) }^2 \leq 2L (f(\xx)- f^\star)\,.  \label{def:lsmooth}
\end{align}
Finally, if $f$ satisfies \eqref{eq:quad-smooth} and is additionally convex, then~\citep[Theorem 2.1.5]{Nesterov2004:book} shows
\begin{align}
 \frac{1}{2L} \norm{\nabla f(\xx) - \nabla f(\yy)}^2 \leq f(\xx) - f(\yy) - \lin{\nabla f(\yy),\xx-\yy}\,, \qquad \forall \xx,\yy \in \R^d. \label{eq:nessmooth}
\end{align}
\end{remark}
\begin{remark}
The Assumptions~\ref{ass:strong} and~\ref{ass:lsmooth} can only be satisfied together if $L \geq \mu$. This can be seen by combining~\eqref{def:strong} with~\eqref{eq:quad-smooth} for $\yy = \xx^\star$:
\begin{align*}
 \frac{\mu}{2}\norm{\xx - \xx^\star}^2 \leq \lin{\nabla f(\xx),\xx-\xx^\star} - f(\xx) + f^\star \leq \frac{L}{2}\norm{\xx - \xx^\star}^2\,.
\end{align*}
\end{remark}

\noindent Lastly, we assume that the noise of the gradient oracle is bounded. Instead of assuming a uniform upper bound, we assume an upper bound of the following form:
\begin{assumption}[$(M,\sigma^2)$-bounded noise]
\label{ass:noise}
For any $\xx$, a gradient oracle of the form $\gg = \nabla f(\xx) + \bxi$ for a differentiable function $f \colon \R^d \to \R$, and conditionally independent noise $\bxi$, there exists two constants $M, \sigma^2\geq 0$, such that
\begin{align}
\Eb{\bxi \mid \xx} &= \0_d \,, & 
 \Eb{\norm{\bxi}^2 \mid \xx} &\leq M\norm{\nabla f(\xx)}^2 + \sigma^2\,. \label{def:noise-general}
\end{align}
\end{assumption}
\noindent We will next state a weaker variant which will be sufficient for quasi-convex functions.

\begin{taggedassumption}{3*}[$(M,\sigma^2)$-bounded noise]\label{ass:noise-weaker}
  For any $\xx$, a gradient oracle of the form $\gg = \nabla f(\xx) + \bxi$ for a $L$-smooth quasi-convex function $f \colon \R^d \to \R$, and conditionally independent noise $\bxi$, there exists two constants $M, \sigma^2\geq 0$, such that
        \begin{align}
            \Eb{\bxi \mid \xx} &= \0_d \,, & 
            \Eb{\norm{\bxi}^2 \mid \xx} &\leq 2LM (f(\xx) - f^\star) + \sigma^2\,.     \label{def:noise}
        \end{align}
\end{taggedassumption}

\begin{remark}\label{rem:weask-assump}
    By combining Assumptions~\ref{ass:lsmooth} and~\ref{ass:noise} it follows for any $\xx$:
    \begin{align*}
        \Eb{\norm{\bxi}^2 \mid \xx} &\leq M\norm{\nabla f(\xx)}^2 + \sigma^2 \stackrel{\eqref{def:lsmooth}}{\leq} 2LM(f(\xx) - f^\star) + \sigma^2. 
    \end{align*}
    This shows that \eqref{def:noise} in Assumption~\ref{ass:noise-weaker} is weaker than \eqref{def:noise-general} in Assumption~\ref{ass:noise}. Though we rely on \eqref{def:noise-general} in our proofs for the sake of conciseness, it is straightforward to adapt our results to the weaker noise condition \eqref{def:noise} for quasi-convex functions. We only need the stronger condition \eqref{def:noise-general} for arbitrary non-convex functions. 
\end{remark}
\begin{remark}
Assumptions~\ref{ass:lsmooth} and~\ref{ass:noise} together imply and an upper bound on the second moment of the gradient oracle of the form
\begin{align}
 \Eb{\norm{\nabla f(\xx) + \bxi}^2 \mid \xx} = \norm{\nabla f(\xx)}^2 +  \Eb{\norm{\bxi}^2 \mid \xx}  \stackrel{\eqref{def:lsmooth},\eqref{def:noise}}{\leq} 2L(1 + M) (f(\xx)-f^\star) + \sigma^2\,. \label{eq:smoothbound}
\end{align}
\end{remark}
\noindent We will now discuss a few examples covered by our assumptions.

\subsection{Key Settings Covered by the Quasi Convexity Assumption}\label{sec:quasi-convex}
Assumption~\ref{ass:strong} does clearly hold for convex and strongly convex  functions, but interestingly also for certain non-convex functions. We also analyze general non-convex functions which do not satisfy Assumption~\ref{ass:strong}, but only prove convergence to a stationary point. 

\paragraph{Quasar convex functions.} 
Functions satisfying Assumption~\ref{ass:strong} are variously called \emph{quasi-strongly convex}~\cite{Necoara2019:linear}, \emph{weakly strongly convex}~\cite{Karimi2016pl} or $(1,\mu)$-\emph{(strongly) quasar-convex} in recent work by~\citet{Hinder2019:quasar}, extending a similar notion previously introduced in~\cite{Hardt2018:weakly}. Our results can be extended to the more general $(\nu,\mu)$-quasar convex functions by following their techniques. The focus of this work is on delayed gradient updates and hence we leave such extensions for future work.

\paragraph{Star (strongly) convex functions.} A notable class of functions satisfying Assumption~\ref{ass:strong} are differentiable star-convex functions. The function $f(x) = \abs{x} \bigl(1-e^{-\abs{x}}\bigr)$ is smooth and star-convex, but not convex~\cite{Nesterov2006:cubic}. We verify Assumption~\ref{ass:strong} by observing
$\lin{ \nabla f(x),x } - f(x) = x^2 e^{- \abs{x} } \geq 0$, that is, equation~\eqref{def:strong} holds for $\mu = 0$. More generally, smooth star convex functions can be constructed by extending an arbitrary smooth positive (but not necessarily convex) function $g \colon \mathbb{S}^{d-1} \to \R$ from the unit sphere sphere to $\R^d$ by e.g.\ setting
\begin{align*}
 f(\xx) = \norm{\xx} \cdot \left(1-e^{-\norm{\xx}}\right) \cdot   g \left(\frac{\xx}{\norm{\xx}}\right)  + \frac{\mu}{2}\norm{\xx}^2\,.
\end{align*}
For other examples and constructions see e.g.~\cite{Lee2016:beyond}.

\subsection{Key Settings Covered by the General Noise Model}
Our Assumption~\ref{ass:noise} on the noise generalizes the usual standard assumptions. The weaker Assumption~\ref{ass:noise-weaker} is more general as we highlight by discussing an inexhaustive list of cases covered.

\paragraph{Uniformly bounded gradients.} A classical assumption in the analysis of stochastic gradient methods is to assume an uniform upper bound on the stochastic gradients, that is $\E \bigl[\norm{\nabla f(\xx) + \bxi}^2 \mid \xx \bigr] \leq G^2$, for a parameter $G^2 \geq 0$, see e.g.~\cite{Nemirovski1983:book,Nemirovski2009:sgd}. This implies that $G^2 \geq \norm{\nabla f(\xx)}^2 + \E\bigl[\norm{\bxi}^2 \mid \xx \bigr]$ for any $\xx$. Thus, even when $\E\bigl[\norm{\bxi}^2 \mid \xx\bigr] = 0$ we have $G^2 > 0$ in general, and thus this assumption is typically too loose to obtain good complexity estimates in the deterministic setting. In contrast Assumption~\ref{ass:noise} is satisfied with $M=0$ and $\sigma^2 = 0$ in the deterministic setting. 

\paragraph{Uniformly bounded noise.} 
Much more fine grained is the uniformly bounded noise assumption, that is assuming $\E \bigl[\norm{\bxi}^2 \mid \xx\bigr] \leq \sigma^2$, see e.g.~\cite{Dekel2012:minibatch}. This setting recovers the deterministic analysis in the case $\sigma^2=0$ and is covered by setting $M=0$ in Assumption~\ref{ass:noise}. The uniformly bounded noise assumption appeared also in~\cite{Arjevani2018:delayed}, thus we extend their analysis not only to a richer function class, but also to (moderately) more general noise models.

\paragraph{Strong-growth condition.}
\citet{Schmidt2013:fastconvergence} introduce the strong-growth condition where it is assumed that there exists a constant $M$ such that $\E \bigl[\norm{\nabla f(\xx) + \bxi}^2 \mid \xx \bigr] \leq M\norm{\nabla f(\xx)}^2$ which translates to Assumption~\ref{ass:noise} with $\sigma^2 =0$. Assumptions~\ref{ass:noise-weaker} with $\sigma^2=0$ is referred to as the \emph{weak growth} condition in~\cite{vaswani2018fast}. These conditions are useful to study benign noise whose magnitude decreases as we get closer to the optimum, and in particular imply that the noise at the optimum is 0.\\

\noindent We now study two settings which particularly benefit from the weaker Assumption~\ref{ass:noise-weaker}, and so apply only to the quasi-convex setting (see Remark~\ref{rem:weask-assump}).
\paragraph{Finite-sum optimization.} In finite sum optimization problems, the objective function can be written as $f(\xx)=\frac{1}{n}\sum_{i=1}^n f_i(\xx)$ for components $f_i \colon \R^d \to \R$, and the gradient oracle is typically just the gradient of one component $\nabla f_i(\xx)$, where the index $i$ is selected uniformly at random from $[n]$. If we assume that each component $f_i$ is convex and satisfies  the smoothness Assumption~\ref{ass:lsmooth} with, for simplicity, the same constant $L$ for each $f_i$, then we can observe:
\begin{align*}
 \EE{i}{\norm{\nabla f_i(\xx) -\nabla f(\xx)}^2} 
 &\leq  \EE{i}{\norm{\nabla f_i(\xx)}^2 } = \EE{i}{\norm{\nabla f_i(\xx) - \nabla f_i(\xx^\star) + \nabla f_i(\xx^\star)}^2 } \\
 &\leq 2 \EE{i}{\norm{\nabla f_i(\xx) - \nabla f_i(\xx^\star)}^2 } +2 \EE{i}{\norm{\nabla f_i(\xx^\star)}^2 }  
 \\    
 & \stackrel{\eqref{eq:nessmooth}}{\leq} \frac{4L}{n}\sum_{i=1}^n\left(f_i(\xx)-f_i(\xx^\star) - \lin{ \nabla f_i(\xx^\star),\xx-\xx^\star}  \right) + 2 \EE{i}{ \norm{\nabla f_i(\xx^\star)}^2} \\
 &= 4L(f(\xx)-f^\star) +  2 \EE{i}{ \norm{\nabla f_i(\xx^\star)}^2 } \,.
\end{align*}
Thus we see that smooth finite sum objectives naturally satisfy the weaker bounded noise Assumption~\ref{ass:noise-weaker}
with parameters $M=2$ and $\sigma^2 = 2 \EE{i}{\norm{\nabla f_i(\xx^\star)}^2}$. These insights in the problem structure were first discussed in \cite{Moulines2011:nonasymptotic} and refined in~\cite{Schmidt2013:fastconvergence,Needell2016:sgd} and allowed to derive the first linear convergence rates for SGD on finite sum problems in the special case when $\sigma^2=0$. Sometimes this special setting is also referred to as the \emph{interpolation setting}~\cite{Ma2018:interpolation}. Closely related is the refined notion of expected smoothness, see the discussions in~\cite{Gower2018:jac,Gower2019:sgd}.

\paragraph{Least-squares.} A classic problem in the literature~\cite{Moulines2011:nonasymptotic} is the least squares minimization problem, where $f(\xx) = \frac{1}{2}\EEb{(\aa,b)\sim \cD}{(b-\lin{\xx,\aa})^2}$ measures the expected square loss over the data samples $(\aa,b)\in \R^d \times \R$, sampled form a (unknown) distribution $\cD$. With the notation $f_{(\aa,b)}:=\frac{1}{2}(b-\lin{\xx,\aa})^2$ the objective $f(\xx)=\EE{(\aa,b)}f_{(\aa,b)}$ takes the standard form of a stochastic optimization problem. An unbiased stochastic gradient oracle for $f$ is given by $\nabla f_{(\aa,b)} (\xx)= -(b-\lin{\xx,\aa})\cdot \aa$.
We observe that
\begin{align*}
 \EE{(\aa,b)} { \norm{\nabla f_{(\aa,b)}(\xx) - \nabla f_{(\aa,b)}(\xx^\star)}^2 } \leq \EEb{(\aa,b)}{2  \lin{\aa,\aa}  \left( f_{(\aa,b)}(\xx) - f_{(\aa,b)}(\xx^\star) \right)}.
\end{align*}
Thus, by assuming a bound on the fourth moment of $\aa$, sometimes written in the form $\Eb{\lin{\aa,\aa} \aa \aa^\top} \preceq R^2 \mathbf{A}$, for a number $R^2$ and Hessian $\mathbf{A}:=\E{\bigl[ \aa \aa^\top\bigr]}$ \cite[cf.][]{Moulines2011:nonasymptotic,Jain2018:parallel}, we can further bound the right hand side by
\begin{align*}
 \EE{(\aa,b)} { \norm{\nabla f_{(\aa,b)}(\xx) - \nabla f_{(\aa,b)}(\xx^\star)}^2 } \leq 2R^2 \left(f(\xx)- f^\star \right)\,.
\end{align*}
Following the same argumentation as outlined for the previous example, we see that least squares optimization under standard assumptions also satisfies the relaxed bounded noise Assumption~\ref{ass:noise-weaker}, with $M$ proportional 
to $3(1 + R^2/L)$ and $\sigma^2$ estimating the noise at the optimum. We see that Assumption~\ref{ass:noise-weaker} in the form of~\eqref{def:noise} is slightly more general, though we like to point out that for least squares problems the analyses typically also make additional assumptions on the structure of covariance of the noise to get more fine-grained results~\cite{Dieuleveut2017:harder,Jain2018:parallel}, a refinement we do not consider here.

\section{Error-Feedback Framework}
\label{sec:technical}
In this section we present a host of lemmas that will ease the presentation of the proofs in the subsequent sections.

\subsection{Error-Compensated and Virtual Sequences}
We will rewrite all algorithms that we consider here in the following, unified, notation, with auxiliary sequences $\{\vv_t\}_{t \geq 0}$ that express the \emph{applied updates}, and $\{\ee_t\}_{t \geq 0}$ that aggregates the delayed information or synchronicity errors. We follow here the ideas from~\cite{Stich2018:sparsified,Karimireddy2019:error} and consider algorithms in the form
\begin{align}
 \begin{split}
  \xx_{t+1} &= \xx_t - \vv_t\,, \\
  \ee_{t+1} &= \ee_t + \gamma_t \gg_t - \vv_t\,.
 \end{split} \tag{EC-SGD} \label{def:EFsgd}
\end{align}
After $T$ such updates, we output $\xx^{\rm out} \in \{\xx_t\}_{t=0}^{T-1}$ where $\xx_t$ is chosen with probability proportional to $w_t$ for some sequence of positive weights $\{w_t\}_{t=0}^{T-1}$.

For instance, for~\eqref{eq:dSGD} we have the updates $\vv_t = \gamma_{t-\tau}\gg_{t-\tau}$ for $t \geq \tau$, and $\vv_t = \0_d$ otherwise; with error terms $\ee_t := \sum_{i=1}^{\tau} \gamma_{t-i}\gg_{t-i}$ (here---for a light notation---we use the convention to only sum over positive indices). 

For the analysis, it will be convenient to define a sequence of `virtual' iterates $\{\tilde \xx \}_{t \geq 0}$. 
That is, the iterates $\tilde \xx_t$ never need to be actually computed, they only appear as a tool in the proof. Formally, we define
\begin{align}
\tilde \xx_{t} &:= \xx_t - \ee_t \,, \qquad \forall t \geq 0, \label{def:tilde}
\end{align}
with $\tilde \xx_0 := \xx_0$ (note that $\ee_0 = \0_d$). 
We observe that
\begin{align}
 \tilde \xx_{t+1} = \xx_{t+1} - \ee_{t+1} = (\xx_t - \vv_t) - (\ee_t + \gamma_t \gg_t - \vv_t) = \tilde \xx_t - \gamma_t \gg_t\,. \label{eq:tilde}
\end{align}

\subsection{A Descent Lemma For Quasi-Convex Functions}
In the next lemma we derive a bound on the one-step progress for the virtual iterates $\tilde \xx_t$ for quasi-convex functions. This proof combines standard techniques~\cite{Nesterov2004:book} with ideas from the perturbed iterate analysis~\cite{Mania2017:perturbed,Leblond2018:improved} and can be seen as an extension of~\cite[Lemma~3.1]{Stich2018:sparsified} to the more general setting considered in this work.

\begin{lemma}
\label{lemma:main}
Let $\{\xx_t,\vv_t,\ee_t\}_{t \geq 0}$ be defined as in~\eqref{def:EFsgd} with gradient oracle $\{\gg_t\}_{t \geq 0}$ and objective function $f \colon \R^d \to \R$ as in Assumptions~\ref{ass:strong}--\ref{ass:noise}. If $\gamma_t \leq \frac{1}{4L(1+M)}$, $\forall t \geq 0$, then for $\{\tilde \xx_t\}_{t \geq 0}$ defined as in~\eqref{def:tilde},
\begin{align}
 \E{ \norm{\tilde \xx_{t+1} - \xx^\star}^2 } &\leq
 \left(1-\frac{\mu \gamma_t}{2}\right) \E{\norm{\tilde \xx_{t} - \xx^\star}^2} 
 - \frac{\gamma_t}{2} \E{(f(\xx_t)- f^\star)} + \gamma_t^2 \sigma^2
 + 3 L \gamma_t   \E{\norm{\xx_t - \tilde \xx_t}^2}\,. \label{eq:main}
\end{align}
\end{lemma}
\begin{proof}
We expand:
\begin{align*}
  \norm{\tilde \xx_{t+1} - \xx^\star}^2  &\stackrel{\eqref{eq:tilde}}{=} \norm{\tilde \xx_{t} - \xx^\star}^2 - 2\gamma_t \lin{\gg_t, \xx_t-\xx^\star} + \gamma_t^2 \norm{\gg_t}^2 + 2\gamma_t \lin{\gg_t, \xx_t - \tilde\xx_t}\,,
\end{align*}
and take expectation w.r.t. the random variable $\bxi_t$:
\begin{align}\begin{split}
  \EEb{\bxi_t} {\norm{\tilde \xx_{t+1} - \xx^\star}^2 \mid \xx_t }
  &\stackrel{(\ref{eq:smoothbound})}{\leq} \norm{\tilde \xx_{t} - \xx^\star}^2 - 2\gamma_t \lin{\nabla f(\xx_t), \xx_t-\xx^\star} + 2L(1+M)\gamma_t^2 (f(\xx_t)-f^\star) \\ &\qquad + \gamma_t^2 \sigma^2  + 2\gamma_t \lin{\nabla f(\xx_t), \xx_t - \tilde \xx_t}\,. \label{eq:long1}
  \end{split}
\end{align} 
By Assumption~\ref{ass:strong}:
\begin{align*}
  -2\lin{\nabla f(\xx_t), \xx_t-\xx^\star} \stackrel{\eqref{def:strong}}{\leq} - \mu \norm{\xx_t- \xx^\star}^2  - 2(f(\xx_t)-f^\star)\,,
\end{align*}
and by $2\lin{\aa,\bb} \leq \alpha \norm{\aa}^2 + \alpha^{-1} \norm{\bb}^2$ for $\alpha > 0$, $\aa,\bb \in \R^d$,
\begin{align*}
2\lin{\nabla f(\xx_t), \tilde \xx_t - \xx_t} \leq \frac{1}{2L} \norm{\nabla f(\xx_t)}^2 + 2L\norm{\xx_t -\tilde \xx_t}^2 \stackrel{\eqref{def:lsmooth}}{\leq} f(\xx_t)-f^\star + 2L \norm{\xx_t -\tilde \xx_t}^2 \,. 
\end{align*}
And by $\norm{\aa + \bb}^2 \leq (1+\beta)\norm{\aa}^2 + (1+\beta^{-1}) \norm{\bb}^2$ for $\beta > 0$ (as a consequence of Jensen's inequality), we further observe
\begin{align*}
 -\norm{\xx_t - \xx^\star}^2 \leq - \frac{1}{2} \norm{\tilde \xx_t - \xx^\star}^2 + \norm{\xx_t - \tilde \xx_t}^2\,.
\end{align*}
Plugging all these inequalities together into~\eqref{eq:long1} yields
\begin{align*}
  \EE{\bxi_t}{\norm{\tilde \xx_{t+1} - \xx^\star}^2}
  &\leq \left(1-\frac{\mu \gamma_t}{2}\right) \norm{\tilde \xx_{t} - \xx^\star}^2 - \gamma_t (1 - 2L(1+M)\gamma_t) (f(\xx_t)- f^\star) \\ &\qquad + \gamma_t^2 \sigma^2 + \gamma_t (2L+\mu) \norm{\xx_t - \tilde \xx_t}^2 \,. 
\end{align*}
The claim follows by the choice $\gamma_t \leq \frac{1}{4L(1+M)}$ and $L \geq \mu$.
\end{proof}

\subsection{A Descent Lemma for Non-Convex Functions}
We now turn our attention to non-convex functions and prove a descent lemma for $\tilde \xx_t$. This proof follows the template of \cite{ghadimi2013stochastic,Karimireddy2019:error} mildly extending the techniques to the general noise condition studied here.

\begin{lemma} \label{lemma:main-nonconvex}
Let $\{\xx_t,\vv_t,\ee_t\}_{t \geq 0}$ be defined as in~\eqref{def:EFsgd} with gradient oracle $\{\gg_t\}_{t \geq 0}$ and objective function $f \colon \R^d \to \R$ satisfying Assumptions~\ref{ass:lsmooth} and~\ref{ass:noise}. If $\gamma_t \leq \frac{1}{2L(1+M)}$, $\forall t \geq 0$, then for $\{\tilde \xx_t\}_{t \geq 0}$ defined as in~\eqref{def:tilde},
\begin{align}
\E[f(\tilde \xx_{t+1})] \leq \E [f(\tilde \xx_t)] - \frac{\gamma_t}{4}\E \norm{\nabla f(\xx_t)}^2 +\frac{\gamma_t^2 L \sigma^2}{2} + \frac{\gamma_t L^2}{2}\E \norm{\xx_t - \tilde \xx_t}^2\,.
\label{eq:main-nonconvex}
\end{align}
\end{lemma}
\begin{proof}
We begin using the definition of $\tilde \xx_{t+1}$ and the smoothness of $f$,
\begin{align*}
    f(\tilde \xx_{t+1}) \stackrel{\eqref{eq:quad-smooth}}{\leq} f(\tilde \xx_t) - \gamma_t \lin{\nabla f(\tilde \xx_t), \gg_t} + \frac{\gamma_t^2 L}{2}\norm{\gg_t}^2\,.
\end{align*}
Taking expectation with respect to $\bxi_t$,
\begin{align*}
    \EEb{\bxi_t}{f(\tilde \xx_{t+1}) | \xx_t} &\stackrel{\eqref{def:noise-general}}{\leq} f(\tilde \xx_t) - \gamma_t \lin{\nabla f(\tilde \xx_t), \nabla f(\xx_t)} + \frac{\gamma_t^2 L(1+M)}{2}\norm{\nabla f(\xx_t)}^2 +\frac{\gamma_t^2 L \sigma^2}{2}\\
    &= f(\tilde \xx_t) - \gamma_t \left(1 - \frac{\gamma_t L(1+M)}{2} \right)\norm{\nabla f(\xx_t)}^2  \\ &\qquad + \gamma_t\lin{\nabla f(\xx_t) - \nabla f(\tilde \xx_t), \nabla f(\xx_t)} +\frac{\gamma_t^2 L \sigma^2}{2}\,.
\end{align*}
Again using $\lin{\aa,\bb} \leq \tfrac{1}{2}\norm{\aa}^2 + \tfrac{1}{2}\norm{\bb}^2$, we can simplify the expression as follows
\begin{align*}
    \lin{\nabla f(\xx_t) - \nabla f(\tilde \xx_t), \nabla f(\xx_t)} &\leq \frac{1}{2}\norm{\nabla f(\xx_t) - \nabla f(\tilde \xx_t)}^2 + \frac{1}{2}\norm{\nabla f(\xx_t)}^2\\
    &\stackrel{\eqref{def:lgradlipschitz}}{\leq}\frac{L^2}{2}\norm{\xx_t - \tilde \xx_t}^2  + \frac{1}{2}\norm{\nabla f(\xx_t)}^2\,.
\end{align*}
Plugging this back, we get our result that
\[
\EEb{\bxi_t}{f(\tilde \xx_{t+1}) | \xx_t} \leq f(\tilde \xx_t) - \frac{\gamma_t \left(1 - \gamma_t L(1+M)\right)}{2}\norm{\nabla f(\xx_t)}^2 + \frac{\gamma_t^2 L \sigma^2}{2} + \frac{\gamma_t L^2}{2}\norm{\xx_t - \tilde \xx_t}^2\,.
\]
Noting that $\gamma_t \leq \frac{1}{2 L(1+M)}$ implies $\frac{\gamma_t(1-\gamma_t L(1+M))}{2} \leq \frac{\gamma_t}{4}$ yields the lemma.
\end{proof}

\subsection{Stepsizes}
We will study SGD with constant stepsizes and slowly decreasing stepsizes in this paper. It will become handy to formalize `slowly decreasing' in the following way.

\begin{definition}[$\tau$-slow sequences]\label{def:slow}
The sequence $\{a_t\}_{t \geq 0}$ of positive values is \emph{$\tau$-slow decreasing} for parameter $\tau \geq 1$ if
\begin{align*}
 a_{t+1} &\leq a_t\,, \qquad \forall t \geq 0, & &\text{and,} & a_{t+1} \left( 1+\frac{1}{2\tau}\right) \geq a_t\,, \qquad \forall t \geq 0.
\end{align*}
The sequence $\{a_t\}_{t \geq 0}$ is \emph{$\tau$-slow increasing} if $\{a_t^{-1}\}_{t \geq 0}$ is $\tau$-slow decreasing.
\end{definition}

\begin{example}
\label{example:slow}
The sequences $\bigl\{u_t := (\kappa + t)^2\bigr\}_{t \geq 0}$, $\bigl\{v_t := \kappa + t\bigr\}_{t \geq 0}$  and $\bigl\{w_t := \bigl(1-\frac{1}{4 c \tau}\bigr)^{-t}\bigr\}_{t \geq 0}$, for $\kappa \geq 8\tau$, $c \geq 1$, and $\tau \geq 1$, are examples of $\tau$-slow increasing sequences.
\end{example}
\begin{proof}
First, consider the sequence $\{u_t\}_{t \geq 0}$. The condition $u_{t+1} \geq u_t$ is easily verified. Furthermore,
$\frac{u_{t+1}}{u_{t}} = \frac{(\kappa + t + 1)^2}{(\kappa + t)^2}\leq \frac{(\kappa + 1)^2}{\kappa^2} = \frac{u_1}{u_0}$, so it suffices to check:
\begin{align*}
 u_1\left(1+\frac{1}{2\tau}\right)^{-1} \leq (\kappa + 1)^2 \left(1-\frac{1}{4\tau} \right) = \kappa^2 +\kappa \underbrace{\left(2-\frac{\kappa}{4\tau}\right)}_{\leq 0} + \underbrace{\left(1-\frac{\kappa}{2\tau}\right)}_{\leq 0} - \frac{1}{4\tau} \leq \kappa^2 =u_0\,, 
\end{align*}
where the first inequality is due to $\bigl(1+\frac{x}{2}\bigr)^{-1} \leq 1-\frac{x}{4}$, for $0 \leq x \leq 1$. This can be verified by $\bigl(1+\frac{x}{2}\bigr)\bigl(1-\frac{x}{4}\bigr) = 1+ \frac{x(2-x)}{8} \geq 1+\frac{x}{8} \geq 1$ for $0 \leq x \leq 1$. 
This shows that $\{u_t\}_{t \geq 0}$ is $\tau$-slow increasing. It follows, that $\{v_t = \sqrt{u_t}\}_{t \geq 0}$ satisfies $v_{t+1} \leq v_{t}\bigl(1+\frac{1}{2\tau}\bigr)^{1/2} \leq v_{t}\bigl(1+\frac{1}{2\tau}\bigr) $ and hence is also $\tau$-slow increasing. For the last sequence we verify that
\begin{align*}
  w_{t+1}^{-1} \left(1+\frac{1}{2\tau}\right) = w_t^{-1} \left(1-\frac{1}{4c \tau}\right) \left(1+\frac{1}{2\tau}\right)\geq  w_t^{-1} \left(1-\frac{1}{4 \tau}\right) \left(1+\frac{1}{2\tau}\right) \geq w_t^{-1} \,,
\end{align*}
in agreement with Definition~\ref{def:slow}.
\end{proof}

\begin{remark} \label{remark:stepsizes}
For stepsizes of the form $\bigl\{\gamma_t = \frac{c}{\kappa + t} \bigr\}_{t \geq 0}$ for a constant $c \geq 0$, $\kappa \geq 8\tau$ and $\tau \geq 1$ it follows from Example~\ref{example:slow} that $\{\gamma_t\}_{t \geq 0}$ and $\{\gamma_t^2\}_{t \geq 0}$ are $\tau$-slow decreasing sequences. Constant stepsizes, $\{\gamma_t = \gamma\}_{t \geq 0}$ for a constant $\gamma > 0$ are $\tau$-slow decreasing for any $\tau \geq 1$.
\end{remark}

\subsection{Technical Lemmas for Deriving the Complexity Estimates}
The next two technical lemmas we borrow from~\cite{Stich2019:sgd}. For completeness we include the proofs as well. Lemma~\ref{lemma:general} is an extension to deal with  arbitrary non-convex functions or function which are quasi-convex with $\mu = 0$.

\begin{lemma}[{\cite[Lemma 7]{Stich2019:sgd}}]
\label{lemma:decreasing}
For decreasing stepsizes $\bigl\{\gamma_t := \frac{2}{a (\kappa + t)} \bigr\}_{t \geq 0}$, and weights $\{w_t := (\kappa + t)\}_{t \geq 0}$ for parameters $\kappa \geq 1$, it holds  for every non-negative sequence $\{r_t\}_{t \geq 0}$  and any  $a > 0$, $c \geq 0$ that
\begin{align*}
 \Psi_T := \frac{1}{W_T}\sum_{t=0}^T \left( \frac{w_t}{\gamma_t} \left(1-a \gamma_t \right) r_t - \frac{w_{t}}{\gamma_t} r_{t+1} + c \gamma_t w_t \right) \leq \frac{a \kappa^2 r_0}{T^2} + \frac{4c}{aT}\,,
\end{align*}
where $W_T := \sum_{i=0}^T w_t$.
\end{lemma}
\begin{proof}
We start by observing that
\begin{align}
 \frac{w_t}{\gamma_t} \left(1-a \gamma_t \right) r_t =  \frac{a}{2} (\kappa + t) (\kappa + t - 2) r_t = \frac{a}{2} \left((\kappa + t - 1)^2 -1 \right) r_t \leq \frac{a}{2} (\kappa + t - 1)^2 r_t\,. \label{eq:424}
\end{align}
By plugging in the definitions of $\gamma_t$ and $w_t$ in $\Psi_T$, we end up with a telescoping sum:
\begin{align*}
\Psi_T \stackrel{\eqref{eq:424}}{\leq} \frac{1}{W_T} \sum_{t=0}^T \left(\frac{a}{2}  (\kappa + t - 1)^2 r_t - \frac{a}{2}(\kappa + t)^2 r_{t+1}  \right) + \sum_{t=0}^T  \frac{2c}{a W_T} \leq \frac{a (\kappa-1)^2 r_0}{2 W_T}  + \frac{2c(T+1)}{a W_T} .
\end{align*}
The lemma now follows with $(\kappa-1)^2 \leq \kappa ^2$, and
 $W_T = \sum_{t=0}^T(\kappa + t) = \frac{(2\kappa + T)(T+1)}{2} \geq \frac{T(T+1)}{2} \geq \frac{T^2}{2}$.
\end{proof}

\begin{lemma}[{\cite[Lemma 2]{Stich2019:sgd}}]
\label{lemma:constant}
For every non-negative sequence $\{r_t\}_{t\geq 0}$ and any parameters $d \geq a > 0$, $c \geq 0$, $T \geq 0$, there exists a constant $\gamma \leq \frac{1}{d}$, such that for constant stepsizes $\{\gamma_t = \gamma\}_{t \geq 0}$ and weights $w_{t}:=(1-a\gamma)^{-(t+1)}$ it holds
\begin{align*}
\Psi_T := \frac{1}{W_T}\sum_{t=0}^T \left( \frac{w_t}{\gamma_t} \left(1-a \gamma_t \right) r_t - \frac{w_{t}}{\gamma_t} r_{t+1} + c \gamma_t w_t \right) = \tilde \cO \left(d r_0 \exp\left[- \frac{aT}{d} \right] + \frac{c}{aT}  \right)\,.
\end{align*}
\end{lemma}
\begin{proof}
By plugging in the values for $\gamma_t$ and $w_t$, we observe that we again end up with a telescoping sum and estimate
\begin{align*}
 \Psi_T = \frac{1}{\gamma W_T} \sum_{t=0}^T \left(w_{t-1}r_t - w_t r_{t+1} \right) + \frac{c\gamma}{W_T}\sum_{t=0}^t w_t \leq \frac{r_0}{\gamma W_T} +  c \gamma \leq \frac{r_0}{\gamma} \exp\left[-a \gamma T \right] + c \gamma\,,
\end{align*}
where we used the estimate $W_T \geq w_{T} \geq (1-a \gamma)^{-T} \geq \exp[a \gamma T]$ for the last inequality. The lemma now follows by carefully tuning $\gamma$. See the proof of Theorem 2 in \cite{Stich2019:sgd}.
\end{proof}

\begin{lemma}\label{lemma:general}
For every non-negative sequence $\{r_t\}_{t\geq 0}$ and any parameters $d \geq 0$, $c \geq 0$, $T \geq 0$, there exists a constant $\gamma \leq \frac{1}{d}$, such that for constant stepsizes $\{\gamma_t = \gamma\}_{t \geq 0}$ it holds:
\begin{align*}
\Psi_T  := \frac{1}{T+1} \sum_{t=0}^T \left( \frac{r_t}{\gamma_t} - \frac{r_{t+1}}{\gamma_t} + c \gamma_t \right) \leq \frac{d r_0}{T+1} + \frac{2\sqrt{c r_0}}{\sqrt{T+1}} \,.
\end{align*}
\end{lemma}
\begin{proof} For constant stepsizes $\gamma_t = \gamma$ we can derive the estimate
\begin{align*}
 \Psi_T = \frac{1}{\gamma (T+1)} \sum_{t=0}^T \left( r_t - r_{t+1} \right) + c \gamma  \leq \frac{r_0}{\gamma(T+1)} + c \gamma \,.
\end{align*}
We distinguish two cases (similar as in~\cite{Arjevani2018:delayed}): if $\frac{r_0}{c (T+1)} \leq \frac{1}{d^2}$, then we chose the stepsize $\gamma = \bigl(\frac{r_0}{c (T+1)}\bigr)^{1/2}$ and get
\begin{align*}
 \Psi_T \leq \frac{2\sqrt{c r_0}}{\sqrt{T+1}} \,,
\end{align*}
on the other hand, if $\frac{r_0}{c (T+1)} > \frac{1}{d^2}$, then we choose $\gamma=\frac{1}{d}$ and get
\begin{align*}
 \Psi_T \leq \frac{d r_0}{T+1} + \frac{c}{d} \leq \frac{d r_0}{T+1} + \frac{\sqrt{c r_0}}{\sqrt{T+1}}\,.
\end{align*}
These two bounds show the lemma.
\end{proof}

\subsection{Technical Lemma for Splitting the Bias and Noise Terms}
For arbitrary vectors $\aa_1,\dots,\aa_\tau \in \R^d$ we have the inequality $\norm{\sum_{i=1}^\tau \aa_i}^2 \leq \tau \sum_{i=1}^\tau \norm{\aa_i}^2$ (as an application of Jensen's inequality). This bound is tight in general (consider $\aa_1 = \aa_2 = \dots = \aa_\tau$). However, for independent zero-mean random variables $\bxi_1,\dots,\bxi_\tau$ we get the much tighter estimate $\E \norm{\sum_{i=1}^\tau \bxi_i}^2 \leq \sum_{i=1}^\tau \E \norm{\bxi_i}^2$. The following lemma combines these two estimates.

\begin{lemma} \label{lem:independent}
Suppose we have a martingale sequence $\aa_1 + \bxi_1,\dots,\aa_{\tau} + \bxi_\tau$ such that for any $ 1 \leq t \leq \tau$, we have $\E[\aa_t + \bxi_t \,|\, \cF_{t-1}] = \aa_t$ conditioned on filtration $\cF_{t-1}$. Then for any $\beta > 0$,
\begin{align}
 \E{\norm{ \sum_{i=1}^\tau \aa_i + \bxi_i}^2} \leq (1+\beta) \tau \sum_{i=1}^\tau \E \norm{\aa_i}^2 + (1+\beta^{-1}) \sum_{i=1}^\tau \E{\norm{\bxi_i}^2}\,. \label{eq:independent}
\end{align} 
\end{lemma}
\begin{proof}
We can simplify as follows:
\begin{align*}
 \E{\norm{\sum_{i=1}^\tau \aa_i + \bxi_i}^2} &\leq (1+\beta)\E\norm{\sum_{i=1}^\tau \aa_i}^2  +  (1+\beta^{-1})\E{\norm{\sum_{i=1}^\tau \bxi_i}^2} \\ &\leq  (1+\beta)\tau \sum_{i=1}^\tau \E\norm{\aa_i}^2 + (1+\beta^{-1})\E{\norm{\sum_{i=1}^\tau \bxi_i}^2}\,,
\end{align*}
where we used $\norm{\sum_{i=1}^\tau \aa_i}^2 \leq \tau \sum_{i=1}^\tau \norm{\aa_i}^2$, and $\norm{\aa + \bb}^2 \leq (1+\beta)\norm{\aa}^2 +  (1+\beta^{-1}) \norm{\bb}^2$ for the inequalities.
The last term can further be simplified as
\begin{align*}
\EE{\bxi_1,\dots,\bxi_\tau}{\norm{\sum_{i=1}^\tau \bxi_i}^2}  
&= \sum_{i,j \in [\tau]} \EE{\bxi_1,\dots,\bxi_\tau} [\bxi_i^\top \bxi_j] = \sum_{i=1}^\tau \EE{\bxi_1,\dots,\bxi_\tau}[\bxi_i^\top \bxi_i]\,. 
\end{align*}
The cross terms (when $i \neq j$) in the last equality evaluate to zero since $\{\bxi_t\}$ form a martingale difference sequence. Taking a full expectation over both sides yields the lemma.
\end{proof}

\section{SGD with Delayed Updates}
\label{sec:delayed}
In this section we analyze SGD with delayed updates~\eqref{eq:dSGD} and extend the results of \cite{Arjevani2018:delayed} from quadratic convex functions to general smooth functions (both quasi-convex and non-convex) through a different proof technique.

SGD with delayed updates, in the form as introduced in~\cite{Arjevani2018:delayed}, can be cast in the~\eqref{def:EFsgd} framework by setting
\begin{align}
 \vv_t &= \begin{cases} \gamma_{t-\tau}\gg_{t-\tau}\,,&\text{if } t \geq \tau, \\ \0_d\,, &\text{if } t<\tau, \end{cases} & \ee_{t} &:= \sum_{i=1}^{\tau} \gamma_{t-i} \gg_{t-i}\,, \label{eq:dsgd_error}
\end{align}
where here $\tau \geq 1$ is an integer \emph{delay}. We use here (and throughout this paper) the convention to only sum over non-negative indices $(t-i) \geq 0$, i.e. the sum in~\eqref{eq:dsgd_error} consists of $\min\{\tau, t\}$ terms. We prove the following rates of convergence:

\begin{theorem}
\label{thm:delay}
Let $\{\xx_t\}_{t \geq 0}$ denote the iterates of delayed stochastic gradient descent~\eqref{eq:dSGD} with constant stepsize $\{\gamma_t = \gamma\}_{t \geq 0}$ on a differentiable function $f \colon \R^d \to \R$ under assumptions Assumptions~\ref{ass:lsmooth} and~\ref{ass:noise}. Then, if $f$ 
\begin{itemize}[nosep,leftmargin=12pt,itemsep=2pt]
 \item satisfies Assumption~\ref{ass:strong} for $\mu > 0$, then there exists a stepsize $\gamma \leq \frac{1}{10L(\tau + M)}$ (chosen as in Lemma~\ref{lemma:constant})
such that
 \begin{align*}
  \E f(\xx^{\rm out})-f^\star = \tilde \cO \left( L(\tau + M) \norm{\xx_0-\xx^\star}^2 \exp \left[- \frac{\mu T}{10L(\tau +M)} \right] + \frac{\sigma^2}{\mu T} \right)\,,
 \end{align*}
 where the output $\xx^{\rm out} \in \{\xx_t\}_{t=0}^{T-1}$ is chosen to be $\xx_t$ with probability proportional to 
 $(1-\mu \gamma/2)^{-t}$. %
 \item satisfies Assumption~\ref{ass:strong} for $\mu = 0$, then there exists a stepsize $\gamma \leq \frac{1}{10L(\tau + M)}$ (chosen as in Lemma~\ref{lemma:general}) such that
 \begin{align*}
 \E f(\xx^{\rm out})-f^\star = \cO \left(  \frac{L(\tau + M) \norm{\xx_0-\xx^\star}^2 }{ T } + \frac{\sigma \norm{\xx_0-\xx^\star} }{\sqrt T}  \right)\,,
 \end{align*}
 where the output $\xx^{\rm out} \in \{\xx_t\}_{t=0}^{T-1}$ is chosen uniformly at random from the iterates $\{\xx_t\}_{t=0}^{T-1}$.
 \item is an arbitrary non-convex function, then there exists a stepsize $\gamma \leq \frac{1}{10 L(\tau + M)}$ (chosen as in Lemma~\ref{lemma:general}), such that
  \begin{align*}
 \E \norm{\nabla f(\xx^{\rm out})}^2 = \cO \left(  \frac{L(\tau + M) (f(\xx_0) - f^\star) }{ T } + \sigma \sqrt{\frac{L(f(\xx_0) - f^\star)}{T}}  \right)\,.
 \end{align*}
 where the output $\xx^{\rm out} \in \{\xx_t\}_{t=0}^{T-1}$ is chosen uniformly at random from the iterates $\{\xx_t\}_{t=0}^{T-1}$.
\end{itemize}
\end{theorem}
\begin{remark}
Proving convergence for a randomly picked iterate $\xx^{\rm out} \in \{\xx_t\}_{t=0}^{T-1}$ is equivalent to show convergence of a (weighted) average of the output criterion, e.g.\ $\frac{1}{T}\sum_{t=0}^{T-1} \E f(\xx_t)-f^\star$ for general quasi-convex functions (and a weighted average for strongly-quasi convex functions).
If in addition convexity holds, our results show convergence of $\E f(\bar \xx_T)-f^\star$, where $\bar \xx_T := \frac{1}{T}\sum_{t=0}^{T-1} \xx_t$ is a (weighted) average of the iterates.
\end{remark}
\begin{remark}\label{rem:decreasingstepsize}
We only consider constant stepsizes in Theorem~\ref{thm:delay}. We can also prove convergence for decreasing stepsize, for instance if $f$ satisfies Assumptions~\ref{ass:strong}--\ref{ass:noise} for $\mu > 0$ and stepsizes are chosen as $\bigl\{\gamma_t = \frac{4}{\mu(\kappa + t)}\bigr\}_{t \geq 0}$  with $\kappa = \frac{40 L(\tau + M)}{\mu}$, then it holds
 \begin{align*}
   \E f(\xx^{\rm out})-f^\star = \cO \left(\frac{L(\tau+M)^2 \norm{\xx_0-\xx^\star}^2 }{\mu T^2}  + \frac{\sigma^2}{\mu T} \right)\,,
\end{align*}
where the output $\xx^{\rm out} \in \{\xx_t\}_{t=0}^{T-1}$ is chosen to be $\xx_t$ with probability proportional to $(\kappa +t)$. 
This rate is dominated by the first claim in Theorem~\ref{thm:delay} when ignoring logarithmic terms (hidden in $\tilde \cO$).
Similar statements can be proven in the settings of Theorem~\ref{thm:sparsified} and~\ref{thm:local} below, but are omitted for brevity.
\end{remark}
\begin{remark}\label{rem:prefactor}
By imposing additionally $T = \Omega\bigl(\frac{L(\tau + M)}{\mu}\bigr)$ the constant in front of the exponential term in Theorem~\ref{thm:delay}
could be improved to $\mu$ instead of $L(\tau + M)$~\citep[see e.g.][Lemma 1]{Karimireddy2019scaffold}.
\end{remark}

Theorem~\ref{thm:delay} shows that only the (asymptotically faster decaying) optimization terms (the terms depending on the initial error $\norm{\xx_0-\xx^\star}$) are impacted by the delay $\tau$; the stochastic terms (the terms depending on $\sigma$), are unaffected by the delay. This means that stochastic delayed gradient descent converges for any constant delay $\tau$ asymptotically at the same rate as stochastic gradient descent without delays. This makes delayed gradient updates a powerful technique to e.g.\ hide communication cost in distributed environments.

\citet{Arjevani2018:delayed} prove for strongly convex quadratic functions and uniformly bounded noise ($M=0$) an upper bound of $\tilde \cO \bigl(L \norm{\xx_0 - \xx^\star}^2 \exp\bigl[- \frac{\mu T}{10 L\tau} \bigr] + \frac{\sigma^2}{\mu T} \bigr)$. We see that the statistical term precisely matches with our result, for the fast decaying exponential term only the prefactors ($L$ vs.\ $L\tau$) are in a slight mismatch%
(this might be an artifact of our proof technique, see also Remark~\ref{rem:prefactor}). Both results imply a $\tilde O \bigl(\frac{\sigma^2}{\mu \epsilon} + \tau \frac{L}{\mu}  \bigr)$ iteration complexity to reach a target accuracy $\epsilon > 0$.
The effect of the delay $\tau$ becomes negligible if the target accuracy is smaller than $\tilde O\bigl(\frac{\sigma^2}{L\tau}\bigr)$, or when the number of iterations $T$ is sufficiently larger than $\tilde \Omega \bigl(\frac{L\tau}{\mu}\bigr)$ (cf.\ the analogous discussion in~\cite{Arjevani2018:delayed}). This is a very mild condition, as we must have $T \geq \tau$ to even observe a single stochastic gradient at the starting point $\xx_0$.

\citet[Theorem 3]{Arjevani2018:delayed} derive a lower bound of $\tilde \Omega\bigl(\tau \sqrt{L/\mu} \bigr)$ in the deterministic ($\sigma^2=0$) setting. It follows that the linear dependence on the delay $\tau$ is optimal and cannot further be improved. Complexity estimates with the square root of the condition number $\frac{L}{\mu}$ are only reached for \emph{accelerated} gradient methods, so it is no surprise that our (non-accelerated) gradient methods does not reach the same complexity. 

For the general quasi-convex case ($\mu = 0$), our upper bound matches precisely the bound in~\cite{Arjevani2018:delayed}, but extends the analysis to non-quadratic functions under Assumption~\ref{ass:strong}. From the iteration complexity $\cO \bigl(\frac{L\tau \norm{\xx_0-\xx^\star}^2}{\epsilon} + \frac{\norm{\xx_0-\xx^\star}^2 \sigma^2}{\epsilon^2} \bigr)$ we see that the effect of the delay becomes negligible if the number of iterations $T$ is sufficiently larger than $\Omega\bigl(\frac{L^2\tau^2\norm{\xx_0-\xx^\star}^2 }{\sigma^2} \bigr)$. Note the quadratic dependence on $\tau$ in the previous condition, as opposed to the strongly convex case where the dependence was only linear. With \emph{accelerated} stochastic methods (e.g.\ \cite{ghadimi2012optimal}), one could hope to attain an improved iteration complexity of $\cO \bigl(\tau\sqrt{\frac{L \norm{\xx_0-\xx^\star}^2}{\epsilon}} + \frac{\norm{\xx_0-\xx^\star}^2 \sigma^2}{\epsilon^2} \bigr)$.
 Such a result would imply that after $\Omega\bigl(\tau^{4/3}\bigl(\frac{L \norm{\xx_0 - \xx^\star}}{\sigma}\bigr)^{2/3}\bigr)$ iterations, the effect of the delay would be negligible. Obtaining such an improvement as well as studying its optimality is left for future work.

For arbitary non-convex functions, we obtain an iteration complexity of $\cO\bigl(\frac{L\tau (f(\xx_0) - f^\star) }{ \epsilon } + \frac{L\sigma^2(f(\xx_0) - f^\star)}{\epsilon^2}  \bigr)$ to reach $\smash{\E \norm{\nabla f(\xx^{\rm out})}^2} \leq \epsilon$. Thus, when the number of iterations is larger than $\Omega \bigl(\frac{L \tau^2 (f(\xx_0) - f^\star)}{\sigma^2} \bigr)$, the first term becomes small and the effect of the delay becomes negligible. This matches with previous analysis of bounded delayed methods for non-convex functions \citep{lian2015asynchronous}. Like in the general quasi-convex case previously studied, the dependence on $\tau$ here is quadratic. However, unlike in the general quasi-convex case, we believe that this is optimal.

For smooth deterministic non-convex functions ($\sigma^2 =0$), \citet{carmon2017lower} show a lower bound of $\Omega(1/\epsilon)$ for \emph{exact} gradient methods to reach an $\epsilon$ stationary point. Thus, unlike in the convex case, the rate cannot be improved using acceleration. We believe that one can prove the iteration complexity of $\cO(\tau/\epsilon)$ shown here is optimal by combining the techniques of \citet[Theorem 3]{Arjevani2018:delayed} and \citet{carmon2017lower}. Further, \cite{arjevani2019complexity} show that the second stochastic term $\cO(\sigma^2/\epsilon^2)$ is also optimal. Together, this shows that our rates are unimprovable for general smooth non-convex functions.

Our analysis here focuses on~\eqref{eq:dSGD}, where the delays are exactly $\tau$. However, it will become clear form the proof that our analysis applies to more general settings, for instance as in~\cite{Feyzmahdavian2016:async}; it suffices that $\tau$ denotes an upper bound on the largest delay.

We provide the proof of Theorem~\ref{thm:delay} in Section~\ref{sec:dsgd_proof1} below. But first, we would like to add a few comments on mini-batch SGD.
\subsection{Mini-Batch SGD} \label{sec:minibatch}
Mini-batch SGD~\cite{Dekel2012:minibatch} is a standard algorithm for distributed optimization, especially large scale machine learning. Each update step is only performed after accumulating a mini-batch of $\tau$ stochastic gradients, all computed with respect to the same point:
\begin{align*}
 \xx_{t+1} &= \begin{cases} \xx_t - \frac{\gamma}{\tau} \sum_{i=0}^{\tau-1} \bigl(\nabla f(\xx_{t-i}) + \bxi_{t-i}\bigr)\,, &\text{if } \tau \vert (t+1) , \\
 \xx_t, &\text{otherwise}. \end{cases}
\end{align*}
We can equivalently write the updates in the~\eqref{def:EFsgd} framework:
\begin{align*}
\vv_{t} &= \begin{cases} \frac{\gamma}{\tau} \sum_{i=0}^{\tau-1} \bigl(\nabla f(\xx_{t-i}) + \bxi_{t-i}\bigr)\,, &\text{if } \tau \vert (t+1), \\
 \0_d, &\text{otherwise}, \end{cases} & 
 \ee_{t} &= \gamma \sum_{\mathclap{i= \lfloor t/\tau \rfloor \cdot \tau}}^{t-1} \bigl( \nabla f(\xx_i) + \bxi_i \bigr) \,.
\end{align*}
\noindent This means, $\xx_{b\tau +0}=\xx_{b\tau +1}=\xx_{(b+1)\tau -1}$ and the error $\ee_{b\tau}=\0_d$ for every integer $b \geq 0$.
\noindent Mini-batch SGD can be seen as a synchronous version of~\eqref{eq:dSGD} in the sense that instead of applying the updates with the constant delay $\tau$, the method waits for $\tau$ gradients to be computed and applies them in a combined update.

Our proof technique also applies to mini-batch SGD (though we will not give the computations in detail here---the proof follows in close analogy to the proof of Theorem~\ref{thm:delay}) and gives a  complexity (number of stochastic gradient computations\footnote{Sometimes the complexity bounds for mini-batch SGD are written in terms of iterations, each comprising $\tau$ stochastic gradient computations (one mini-batch). The complexity estimates in number of \emph{iterations} translate to $\tilde \cO \bigl(\frac{\sigma^2}{\mu \tau \epsilon} + \frac{L}{\mu} \bigr)$, $\cO \bigl(\frac{L \norm{\xx_0-\xx^\star}^2}{\epsilon} + \frac{\norm{\xx_0-\xx^\star}^2 \sigma^2}{\tau \epsilon^2} \bigr)$ and $\cO\bigl(\frac{L (f(\xx_0) - f^\star) }{ \epsilon } + \frac{L\sigma^2(f(\xx_0) - f^\star)}{\tau\epsilon^2}  \bigr)$ respectively. }) on quasi-convex functions of $\tilde O \bigl(\frac{\sigma^2}{\mu \epsilon} + \tau \frac{L}{\mu}\bigr)$ when $\mu > 0$ and $\cO \bigl(\frac{L\tau \norm{\xx_0-\xx^\star}^2}{\epsilon} + \frac{\norm{\xx_0-\xx^\star}^2 \sigma^2}{\epsilon^2} \bigr)$ when $\mu = 0$. For general non-convex functions we obtain iteration complexity of $\cO\bigl(\frac{L\tau (f(\xx_0) - f^\star) }{ \epsilon } + \frac{L\sigma^2(f(\xx_0) - f^\star)}{\epsilon^2}  \bigr)$. These bounds match with the known bounds for this method (up to logarithmic factors)  \cite[see e.g.][]{Dekel2012:minibatch,Bottou2018:book,Stich2019:sgd}. Moreover, we see that the complexity of the delayed method matches with mini-batch SGD.

\subsection{Proof of Theorem~\ref{thm:delay}}
\label{sec:dsgd_proof1}
We start with a key lemma where we derive an upper bound on $\E \norm{\ee_t}^2$.

\begin{lemma} \label{lemma:dsgd_final}
Let $\{\ee_t\}_{t \geq 0}$ be defined as in~\eqref{eq:dsgd_error} for stepsizes $\{\gamma_t\}_{t \geq 0}$ with $\gamma_{t} \leq \frac{1}{10 L(\tau+M)}$, $\forall t \geq 0$ and $\{\gamma_t^2\}_{t \geq 0}$ $\tau$-slow decaying. Then
\begin{align}
 \Eb{ 3L \norm{\ee_t}^2 } \leq \frac{1}{10 L \tau}  \sum_{i=1}^\tau  \E{\norm{\nabla f(\xx_{t-i})}^2}  + 2\gamma_t \sigma^2 \,. \label{eq:dsgd_bound1}
\end{align}
Furthermore, for any $\tau$-slow increasing sequence $\{w_t\}_{t \geq 0}$ of non-negative values it holds:
\begin{align}
    3 L\sum_{t=0}^T w_t \E \norm{\ee_t}^2 \leq \frac{ 1 }{5L} \sum_{t=0}^T w_t \left(\E \norm{\nabla f(\xx_{t-i})}^2 \right) +  2\sigma^2 \sum_{t=0}^T w_t \gamma_t\,. \label{eq:dsgd_bound2}
\end{align}
\end{lemma}
\begin{proof}
We start with the first claim. By definition and Lemma~\ref{lem:independent} from above (with $\beta = \frac{1}{2}$), we have the bound:
\begin{align*}
 \E { \norm{\ee_t}^2 } &= \E { \norm{\sum_{i=1}^\tau \gamma_{t-i}\bigl(\nabla f(\xx_{t-i}) + \bxi_{t-i})\bigr)}^2 } \\  &\stackrel{\eqref{eq:independent}}{\leq} \frac{3}{2} \min\{\tau,t\} \sum_{i=1}^\tau \gamma_{t-i}^2 \norm{\nabla f(\xx_{t-i})}^2 +  3 \sum_{i=1}^\tau \gamma_{t-i}^2 \E \norm{\bxi_{t-i}}^2 \\
 & \stackrel{\eqref{def:noise-general}}{\leq} \frac{3}{2} (\tau + M) \sum_{i=1}^\tau  \gamma_{t-i}^2 \norm{\nabla f(\xx_{t-i})}^2 + 3 \sigma^2 \sum_{i=1}^\tau \gamma_{t-i}^2\,.
\end{align*}
For $i \leq \tau$ we have the upper bound $\gamma_{t-i}^2 \leq \gamma_t^2 \bigl(1+\frac{1}{2\tau}\bigr)^\tau \leq \gamma_t^2 \exp \left[\frac{\tau}{2\tau} \right] \leq 2 \gamma_t^2$, as $1+x \leq e^x$, $\forall x \in \R$. Thus we can simplify:
\begin{align*}
 \E {\norm{\ee_t}^2} \leq \gamma_t^2   (\tau + M) \sum_{i=1}^\tau \left( 3 \norm{\nabla f(\xx_{t-i})}^2 + 6 \min\{\tau,t\} \sigma^2 \right)\,.
\end{align*}
By observing that the choice 
$\gamma_t \leq \frac{1}{10 L(\tau+M)}$ implies  
$\bigl(3L \cdot 3(\tau + M) \gamma_t^2\bigr) \leq \frac{1}{10L(\tau + M)} \leq \frac{1}{10 L \tau}$ and 
$\bigl(3 L \cdot 6\tau \gamma_t\bigr) \leq \frac{2L\tau}{L(\tau +M)} \leq 2$ we show the first claim. %

For the second claim, we observe that for $\tau$-slow increasing $\{w_t\}_{t\geq 0}$ we have $w_t \leq w_{t-i} \bigl(1+\frac{1}{2\tau}\bigr)^{i} \leq w_{t-i} \bigl(1+\frac{1}{2\tau}\bigr)^{\tau} \leq w_{t-i}\exp\left[\frac{1}{2} \right] \leq 2w_{t-i}$ for every $0 \leq i \leq \tau$. Thus we can estimate
\begin{align*}
3 L\sum_{t=0}^T w_t \E \norm{\ee_t}^2 & \stackrel{\eqref{eq:dsgd_bound1}}{\leq} \frac{1}{5 L} \sum_{t=0}^T  \frac{ w_t }{2 \tau} \sum_{i=1}^\tau \left(\E \norm{\nabla f(\xx_{t-i})}^2 \right) + 2\sigma^2 \sum_{t=0}^T w_t \gamma_t \\
&\leq \frac{1}{5L} \sum_{t=0}^T  \frac{ 1 }{\tau} \sum_{i=1}^\tau w_{t-i} \left(\E \norm{\nabla f(\xx_{t-i})}^2 \right) + 2\sigma^2 \sum_{t=0}^T w_t \gamma_t \\
&\leq \frac{ 1 }{5L} \sum_{t=0}^T w_t \left(\E \norm{\nabla f(\xx_{t-i})}^2 \right) +  2\sigma^2 \sum_{t=0}^T w_t \gamma_t  \,.
\end{align*}
This concludes the proof.
\end{proof}

This lemma, together with the estimate on the one step progress derived in Lemma~\ref{lemma:main} for the convex case and Lemma~\ref{lemma:main-nonconvex} for the non-convex case, allows us to obtain a recursive description of the suboptimality. To obtain the complexity estimates, we follow closely the technique outlined in \cite{Stich2019:sgd}.

\begin{proof}[Proof of Theorem~\ref{thm:delay}]
We split the proof in two parts. First dealing with the three cases under the quasi-convexity Assumption~\ref{ass:strong}, and then addressing the remaining case without Assumption~\ref{ass:strong}.
\paragraph{Quasi convex functions (claims 1--2, and Remark~\ref{rem:decreasingstepsize}).}
Lemma~\ref{lemma:dsgd_final} together with \eqref{def:lsmooth} gives
\[
    3 L\sum_{t=0}^T w_t \E \norm{\ee_t}^2 \leq \frac{ 2 }{5} \sum_{t=0}^T w_t \left(\E f(\xx_{t-i}) - f^\star \right) + 2\sigma^2 \sum_{t=0}^T w_t \gamma_t\,.
\]
We would like to remark that we could have written Lemma~\ref{lemma:dsgd_final} directly in the form as given above, relying on only Assumption~\ref{ass:noise-weaker} instead of the stronger Assumption~\ref{ass:noise} (see Remark~\ref{rem:weask-assump}). However, this would be insufficient for proving convergence to stationary points for arbitrary non-convex functions (claim 4) as we will see later.

Observe that the conditions of Lemma~\ref{lemma:main} are satisfied. With the notation $r_t := \E{\norm{\tilde \xx_{t+1}-\xx^\star}^2}$ and $s_t := \E{f(\xx_t)-f^\star}$ we thus have for any $w_t > 0$:
\begin{align*}
 \frac{w_t}{2} s_t \stackrel{\eqref{eq:main}}{\leq} \frac{w_t}{\gamma_t} \left(1-\frac{\mu \gamma_t}{2}\right) r_t - \frac{w_{t}}{\gamma_t} r_{t+1} + \gamma_t w_t \sigma^2 + 3 w_t L \E{\norm{\ee_t}^2}\,.
\end{align*}
Suppose now---we show this below---that the conditions of Lemma~\ref{lemma:dsgd_final} are satisfied. With this lemma we have:
\begin{align*}
 \frac{1}{2} \sum_{t=0}^T w_t s_t \leq \sum_{t=0}^T \left( \frac{w_t}{\gamma_t} \left(1-\frac{\mu \gamma_t}{2}\right) r_t - \frac{w_{t}}{\gamma_t} r_{t+1} + 3 \gamma_t w_t \sigma^2 \right) + \frac{2}{5}\sum_{t=0} ^T w_t s_t  \,.
\end{align*}
This can be rewritten as
\begin{align*}
 \frac{1}{W_T} \sum_{t=0}^T w_t s_t \leq  \frac{10}{W_T} \sum_{t=0}^T \left( \frac{w_t}{\gamma_t} \left(1-\frac{\mu \gamma_t}{2}\right) r_t - \frac{w_{t}}{\gamma_t} r_{t+1} + 3 \gamma_t w_t \sigma^2 \right) =: \Xi_T \,.
\end{align*}
All that is remaining is to check that the conditions of Lemma~\ref{lemma:dsgd_final} are indeed satisfied, and to derive an estimate on $\Xi_T$. For this, we discuss the two cases of the theorem separately.

For the first claim (constant stepsize, $\mu > 0$), the conditions of Lemma~\ref{lemma:dsgd_final} are easy to check, as $\gamma \leq \frac{1}{10L(\tau + M)}$ by definition. We further observe that $\bigl(1-\frac{\mu\gamma}{2}\bigr) \geq \bigl(1- \frac{\mu}{20L(\tau + M)} \bigr) \geq \bigl(1-\frac{1}{8\tau}\bigr)$ and by Example~\ref{example:slow} it follows that weights chosen as $w_t = (1-\frac{\mu \gamma}{2})^{-(t+1)}$ are $2\tau$-slow increasing (and hence also $\tau$-slow increasing). The claim follows by Lemma~\ref{lemma:dsgd_final}, Lemma~\ref{lemma:constant} and observing that the indicated sampling probability for choosing $\xx^{\rm out}$ out of $\{\xx_t\}_{t=0}^{T-1}$ matches with the chosen weights $w_t$.

For the second claim (constant stepsize, $\mu = 0$), we invoke Lemma~\ref{lemma:general} with weights $\{w_t = 1\}_{t \geq 0}$.

Finally, for the stepsizes $\gamma_t = \frac{4}{\mu (\kappa + t)}$ as used in Remark~\ref{rem:decreasingstepsize}, we observe $\gamma_t \leq \gamma_0 = \frac{4}{\mu \kappa} \leq \frac{1}{10L(\tau + M)}$, by the choice of $\kappa$. In Remark~\ref{remark:stepsizes} we have further shown that $\{\gamma_t^2\}_{t\geq 0}$ is $\tau$ slow decreasing, as $\kappa \geq 8\tau$ (and $\tau \geq 1$). Furthermore, in Example~\ref{example:slow} we have show that the weights $\{w_t = \kappa + t\}_{t \geq 0}$ are $2\tau$-slow increasing for $\kappa \geq 16\tau$ (and hence they are also $\tau$-slow increasing).
Thus, the conditions of Lemma~\ref{lemma:dsgd_final} are indeed satisfied and the technical Lemma~\ref{lemma:decreasing} provides the claimed upper bound on $\Xi_T$.

\paragraph{Arbitrary smooth non-convex functions (claim 3).}
For the non-convex case, Lemma~\ref{lemma:main-nonconvex} gives us the progress of one step. Using notation $r_t := 4\E (f(\tilde \xx_{t}) - f^\star)$, $s_t := \E \norm{\nabla f(\xx_t)}^2$, $c = 4L \sigma^2$, and $w_t = 1$ we have
\begin{align*}
    \frac{1}{4 W_T} \sum_{t=0}^T w_t s_t &\stackrel{\eqref{eq:main-nonconvex}}{\leq} \frac{1}{W_T}\sum_{t=0}^T w_t\left(\frac{r_t}{4\gamma_t} - \frac{r_{t+1}}{4\gamma_t} + \frac{\gamma_t c}{8}\right) + \frac{L^2}{2W_T} \sum_{t=0}^T w_t\E \norm{\xx_t - \tilde \xx_t}^2\\
    &\stackrel{\eqref{eq:dsgd_bound2}}{\leq} \frac{1}{W_T}\sum_{t=0}^T w_t\left(\frac{r_t}{4\gamma_t} - \frac{r_{t+1}}{4\gamma_t} + \frac{\gamma_t c}{8}\right) + \frac{L^2}{2W_T} \sum_{t=0}^T\left( \frac{1}{15 L^2 }w_t s_t + \frac{w_t\gamma_t c}{4 L^2}\right)\,.
\end{align*}
The above equation can be simplified as: 
\[
    \frac{1}{5 W_T} \sum_{t=0}^T w_t s_t \leq \frac{1}{W_T}\sum_{t=0}^T w_t\left(\frac{r_t}{4\gamma_t} - \frac{r_{t+1}}{4\gamma_t} + \frac{\gamma_t c}{4}\right)\,.
\]
We can now invoke Lemma~\ref{lemma:general} with weights $\{w_t = 1\}_{t \geq 0}$ to finish the proof.
\end{proof}

\section{Error Compensated SGD with Arbitrary Compressors}
\label{sec:efsgd}
In this section we analyze SGD with error-feedback (or error-compensation) and generalize and improve the results of~\cite{Stich2018:sparsified, Karimireddy2019:error} through the more refined analysis developed here. This method is of particular importance in distributed optimization to reduce communication costs, but we consider only the single worker case here.

\noindent We consider algorithms  that take the following  form in the~\eqref{def:EFsgd} framework:
\begin{align}
 \vv_t &:= \cC(\ee_{t} + \gamma_{t} \gg_{t} )\,, & \ee_{t+1} := \ee_{t} + \gamma_{t} \gg_{t} - \vv_t\,, \label{eq:sparse_error}
\end{align} 
where here $\cC$ denotes a $\delta$-\emph{approximate compressor} or $\delta$-compressor for short.
\begin{definition}[$\delta$-approximate compressor]
A random operator $\cC \colon \R^d \to \R^d$ that satisifes for a parameter $\delta > 0$:
\begin{align}
 \EE{\cC} {\norm{\xx-\cC(\xx)}^2} \leq (1-\delta)\norm{\xx}^2\,,\qquad \forall \xx \in \R^d. \label{eq:compressor}
\end{align}
\end{definition}

\noindent In contrast to~\eqref{eq:dSGD} studied in the previous section, we do not precisely know the explicit structure of $\ee_t$ here (e.g.\ if it can be written as the sum of $\tau$ stochastic gradient estimators). Instead, for $\delta$\nobreakdash-approximate compressors, we only know an upper bound on the squared norm $\norm{\ee_t}^2$. The notion~\eqref{eq:sparse_error} comprises a much richer class of algorithms. 
For illustration---and to highlight the connection the previous section---consider the compressor $\cC_\tau$, defined for a parameter $\tau \geq 1$ as:
\begin{align*}
 \cC_\tau (\xx) = \begin{cases} \0_d\,, &\text{with probability }1-\frac{1}{\tau}\,,\\
 \xx\,, &\text{with probability }\frac{1}{\tau}\,. \end{cases}
\end{align*}
This operator is a $\delta = \frac{1}{\tau}$ compressor. Moreover, $\ee_t$ can  be written as the sum of (in expectation) $\tau$  stochastic gradients. Thus, we would expect the algorithm~\eqref{eq:dSGD} to behave similarly as algorithm \eqref{eq:sparse_error} with a $\delta = \frac{1}{\tau}$ compressor. Indeed, this intuition is true and Theorem~\ref{thm:sparsified} can be seen as a generalization of Theorem~\ref{thm:delay} with $\tau$ replaced by $\frac{2}{\delta}$. The following operator
\begin{align*}
 \bigl[\cS_\delta (\xx)\bigr]_i &= \begin{cases} 0\,, &\text{with probability }1-\delta\,,\\
 [\xx]_i\,, &\text{with probability }\delta \,,\end{cases} 
\end{align*}
where $[\xx]_i:=\lin{\xx,\ee_i}$, is also a $\delta$-approximate compressor. This shows that this notion not only comprises delayed gradients, but also delayed (atomic) block-coordinates updates. We refer e.g.\ to \cite{Alistarh2018:topk,Stich2018:sparsified,
cordonnier2018:sparse,Karimireddy2019:error} for the discussion of key examples, including sparsification and quantization (that we have not discussed here).
However, it is important to note that whilst our framework~\eqref{def:EFsgd} covers distributed SGD implementations in general (see for instance Section~\ref{sec:minibatch}), we here analyze error compensated SGD only for the special case of a single-machine implementation and our results do no apply to a fully distributed optimization setting as for instance considered in~\cite{cordonnier2018:sparse}. Fully compressed communication in such a distributed setting involves each machine maintaining a separate error vector, and hence does not directly fit in our framework~\eqref{def:EFsgd}.

\begin{theorem}
\label{thm:sparsified}
Let $\{\xx_t\}_{t \geq 0}$ denote the iterates of the error compensated stochastic gradient descent~\eqref{eq:sparse_error} 
with constant stepsize $\{\gamma_t = \gamma\}_{t \geq 0}$ and
with a $\delta$-approximate compressor on a differentiable function $f \colon \R^d \to \R$ under assumptions Assumptions~\ref{ass:lsmooth} and~\ref{ass:noise}. Then, if $f$ 
\begin{itemize}[nosep,leftmargin=12pt,itemsep=2pt]
 \item satisfies Assumption~\ref{ass:strong} for $\mu > 0$, then there exists a stepsize $\gamma \leq \frac{1}{10L(2/\delta + M)}$ (chosen as in Lemma~\ref{lemma:constant})
such that
 \begin{align*}
  \E f(\xx^{\rm out})-f^\star = \tilde \cO \left( L(1/\delta + M) \norm{\xx_0-\xx^\star}^2 \exp \left[- \frac{\mu T}{10L(2/\delta + M)} \right] + \frac{\sigma^2}{\mu T} \right)\,,
 \end{align*}
 where the output $\xx^{\rm out} \in \{\xx_t\}_{t=0}^{T-1}$ is chosen to be $\xx_t$ with probability proportional to 
 $(1-\mu \gamma/2)^{-t}$. 
 \item satisfies Assumption~\ref{ass:strong} for $\mu = 0$, then there exists a stepsize $\gamma \leq \frac{1}{10L(2/\delta + M)}$ (chosen as in Lemma~\ref{lemma:general}) such that
 \begin{align*}
 \E f(\xx^{\rm out})-f^\star = \cO \left(  \frac{L(1/\delta + M) \norm{\xx_0-\xx^\star}^2 }{ T } + \frac{\sigma \norm{\xx_0-\xx^\star} }{\sqrt T}  \right)\,,
 \end{align*}
  where the output $\xx^{\rm out} \in \{\xx_t\}_{t=0}^{T-1}$ is chosen uniformly at random from the iterates $\{\xx_t\}_{t=0}^{T-1}$.
 \item is an arbitrary non-convex function, then there exists a stepsize $\gamma \leq \frac{1}{10L(1/\delta + M)}$ (chosen as in Lemma~\ref{lemma:general}), such that
  \begin{align*}
 \E \norm{\nabla f(\xx^{\rm out})}^2 = \cO \left(  \frac{L(1/\delta + M) (f(\xx_0) - f^\star) }{ T } + \sigma \sqrt{\frac{L(f(\xx_0) - f^\star)}{T}}  \right)\,.
 \end{align*}
 where the output $\xx^{\rm out} \in \{\xx_t\}_{t=0}^{T-1}$ is chosen uniformly at random from the iterates $\{\xx_t\}_{t=0}^{T-1}$.
\end{itemize}
\end{theorem}
In analogy to Theorem~\ref{thm:delay}, this result shows that the stochastic terms (the ones depending on $\sigma$) in the rate are not affected by the $\delta$ parameter.
\citet{Stich2018:sparsified} proved under the bounded gradient assumption for strongly convex functions an upper bound of $\cO \bigl(\smash{\frac{\mu}{\delta^3 T^3} \norm{\xx_0- \xx^\star}^2 + \frac{ L G^2}{\mu^2 \delta^2 T^2}   + \frac{G^2}{\mu T} }\bigr)$ 
which implies an iteration complexity of $\Omega\bigl(\smash{\frac{G^2}{\mu \epsilon} + \frac{G\sqrt{L}}{\mu \delta \sqrt{\epsilon}} } \bigr)$. Our second result implies iteration complexity $\tilde \cO \bigl( \smash{ \frac{\sigma^2}{\mu \epsilon} + \frac{L}{\mu \delta} }\bigr)$ which is strictly better, as $\sigma^2 \leq G^2$ in general. Moreover, as in the previous section we notice that the impact of the compression becomes negligible if $\epsilon$ is smaller than $\tilde \cO \bigl(\frac{\delta \sigma^2}{L} \bigr)$ or $T = \tilde \Omega \bigl(\frac{L}{\mu \delta} \bigr)$. As only for $T=\Omega\bigl(\frac{1}{\delta}\bigr)$ the first gradient is fully received (in expectation), this is again a very mild condition that implies \emph{compression for free}, i.e.\ without increasing the iteration complexity. By the same arguments as in the previous section and considering the operator $\cC_\tau$ from above, we see that the linear dependency on $\frac{1}{\delta}$ cannot further be improved in general. The general quasi convex setting ($\mu = 0$) was not studied in previous work.

Under the bounded gradient assumption, \citet{Karimireddy2019:error} prove that that the gradient norm converges at a rate of $\cO\bigl(G\sqrt{\frac{(f(\xx_0) - f^\star)}{T}} + \frac{L G^2}{\delta^2 T}\bigr)$. This translates to an iteration complexity of $\cO\bigl( \smash{\frac{L(f(\xx_0) - f^\star) }{\delta^2 \epsilon } + \frac{LG^2(f(\xx_0) - f^\star)}{\epsilon^2} } \bigr)$. In contrast, our rates give an iteration complexity of $\cO\bigl(\frac{L(f(\xx_0) - f^\star) }{\delta \epsilon } + \frac{L\sigma^2(f(\xx_0) - f^\star)}{\epsilon^2}  \bigr)$. Thus we improve in two regards: first our rates replace the second moment bound with a variance bound, and second we obtain a linear dependence on $\delta$ instead of the quadratic dependence in \cite{Karimireddy2019:error}.

Our improved rates are partially due by relaxing the bounded gradient assumption, and also a more careful bound on the error term.

\subsection{Proof of Theorem~\ref{thm:sparsified}}

We follow a similar structure as in the proof of Theorem~\ref{thm:delay}. We first derive an an upper bound on $\E \norm{\ee_t}^2$.

\begin{lemma} \label{lemma:sparse_final}
Let $\ee_t$ be as in~\eqref{eq:sparse_error} for a $\delta$-approximate compressor $\cC$ and stepsizes $\{\gamma_t\}_{t \geq 0}$ with $\gamma_{t+1} \leq \frac{1}{10 L (2/\delta+M)}$, $\forall t \geq 0$ and $\{\gamma_t^2\}_{t \geq 0}$ $\frac{2}{\delta}$-slow decaying. Then
\begin{align}
 \Eb{ 3L \norm{\ee_{t+1}}^2 } \leq \frac{\delta}{64 L}\sum_{i=0}^t \left(1-\frac{\delta}{4}\right)^{t-i} \left(\E \norm{\nabla f(\xx_{t-i})}^2 \right) +  \gamma_t  \sigma^2\,. \label{eq:sparse_bound1}
\end{align}
Furthermore, for any $\frac{4}{\delta}$-slow increasing non-negative sequence $\{w_t\}_{t \geq 0}$ it holds:
\begin{align*}
    3 L\sum_{t=0}^T w_t \E \norm{\ee_t}^2 \leq \frac{1}{8L} \sum_{t=0}^T w_t \left(\E \norm{\nabla f(\xx_{t-i})}^2 \right) +  \sigma^2 \sum_{t=0}^T w_t \gamma_t\,. %
\end{align*}
\end{lemma} 

\begin{proof}
We start with the first claim. By definition of the error sequence
\begin{align*}
\EE{\cC}{ \norm{\ee_{t+1}}^2 } &\stackrel{\eqref{eq:sparse_error}}{=} \EE{\cC}{\norm{\ee_{t} + \gamma_t \gg_t - \cC(\ee_{t} + \gamma_t \gg_t)}^2} \stackrel{\eqref{eq:compressor}}{\leq} (1-\delta) \norm{\ee_{t} + \gamma_t \gg_t}^2\,.
\end{align*}
We further simplify with Assumption~\ref{ass:noise}:
\begin{align*}
 \EE{\bxi_t}{ \norm{\ee_{t+1}}^2} &\leq (1-\delta) \EE{\bxi_t}{ \norm{\ee_t + \gamma_t (\nabla f(\xx_t) + \bxi_t)}^2} \\
 &\stackrel{\eqref{def:noise-general}}{=} (1-\delta) \norm{\ee_t + \gamma_t \nabla f(\xx_t)}^2 + \gamma_t^2 (1-\delta) \EE{\bxi_t}{\norm{\bxi_t}^2} \\
 &\stackrel{\eqref{def:noise-general}}{\leq} (1-\delta) \norm{\ee_t + \gamma_t \nabla f(\xx_t)}^2 + \gamma_t^2 (1-\delta) \left(M\norm{\nabla f(\xx_t)}^2 + \sigma^2 \right) \,. 
\end{align*}
With $\norm{\aa + \bb}^2 \leq (1+\beta)\norm{\aa}^2 + (1+\beta^{-1})\norm{\bb}^2$, for any $\beta \geq 0$, we continue:
\begin{align*}
 \EE{\bxi_t}{ \norm{\ee_{t+1}}^2} & \leq (1-\delta)(1+\beta) \norm{\ee_t}^2 
+ \gamma_t^2 (1-\delta)(1+1/\beta) \norm{\nabla f(\xx_t)}^2  \\ & \qquad  + \gamma_t^2 (1-\delta) \left(M\norm{\nabla f(\xx_t)}^2 + \sigma^2 \right) \\
& \stackrel{\eqref{def:lsmooth}}{\leq} (1-\delta)(1+\beta) \norm{\ee_t}^2 + \gamma_t^2 ( 1+1/\beta + M) \norm{\nabla f(\xx_t)}^2 +   \gamma_t^2 \sigma^2   \,,
\end{align*}
where we used the smoothness Assumption~\ref{ass:lsmooth} and dropped---for convenience---one superfluous $(1-\delta)$ factor in the last inequality. By unrolling the recurrence, and picking $\beta = \frac{\delta}{2(1-\delta)}$, such that $(1+1/\beta) = (2-\delta)/\delta \leq 2/\delta$, and $(1-\delta)(1+\beta) \leq (1-\delta/2)$  we find
\begin{align*}
\E{\norm{\ee_{t+1}}^2} & \leq \sum_{i=0}^t \gamma_i^2 \left[(1-\delta)(1+\beta)  \right]^{t-i}  \left(( 1+1/\beta + M) \E \norm{\nabla f(\xx_i)}^2 +  \sigma^2\right) \\
& \leq \sum_{i=0}^t  \gamma_i^2 \left(1-\frac{\delta}{2}\right)^{t-i}  \left( \left( \frac{2}{\delta} + M \right) \E \norm{\nabla f(\xx_i)}^2 +  \sigma^2 \right)\,.
\end{align*} 
For $\frac{2}{\delta}$-slow decreasing $\{\gamma_t^2\}_{t \geq 0}$ it holds  $\gamma_{i}^2\leq \gamma_t^2 \bigl(1+\frac{\delta}{4} \bigr)^{t-i}$. As $(1-\delta/2)(1+\delta/4)\leq (1-\delta/4)$, we continue:
\begin{align*}
\E{\norm{\ee_{t+1}}^2} &\leq \sum_{i=0}^t \gamma_t^2 \left(1+\frac{\delta}{4}\right)^{t-i} \left(1-\frac{\delta}{2}\right)^{t-i}  \left( \frac{2}{\delta} + M \right) \norm{\nabla f(\xx_i)}^2  + \gamma_t^2 \sum_{i=0}^t  \left(1+\frac{\delta}{4}\right)^{t-i}  \sigma^2 \\
&\leq \gamma_t^2 \sum_{i=0}^t \left(1-\frac{\delta}{4}\right)^{t-i}  \left( \frac{2}{\delta} + M \right)\E \norm{\nabla f(\xx_i)}^2  + \gamma_t^2 \frac{4 \sigma^2}{\delta}\,.
\end{align*}
By observing that the choice of the stepsize $\gamma_t \leq \frac{1}{10 L(2/\delta+M)}$ implies $\bigl( 3L\cdot (2/\delta+M)\gamma_t^2\bigr)\leq \frac{\delta}{64L}  $ and $\bigl(3L \cdot 4/\delta \gamma_t\bigr) \leq 1$ we prove the first claim.

For the second claim, we observe that for $\frac{4}{\delta}$-slow increasing $\{w_t\}_{t \geq 0}$ we have $w_t \leq w_{t-i} \bigl(1+\frac{\delta}{8}\bigr)^i$. Hence,
\begin{align*}
 3L \sum_{t=0}^T w_t \E{\norm{\ee_{t}}^2} &\stackrel{\eqref{eq:sparse_bound1}}{\leq}  \frac{\delta}{64L}  \sum_{t=0}^T \sum_{i=0}^{t-1} w_t \left(1-\frac{\delta}{4}\right)^{t-i}  \E \norm{\nabla f(\xx_i)}^2 +  w_t  \gamma_{t-1}  \sigma^2 \\
 &\leq  \frac{ \delta}{64L} \sum_{t=0}^T  \sum_{i=0}^{t-1} w_{i} \left(1+\frac{\delta}{8}\right)^{t-i} \left(1-\frac{\delta}{4}\right)^{t-i} \E \norm{\nabla f(\xx_i)}^2 + \sigma^2 \sum_{t=0}^T w_t \gamma_t \\
 &\leq \frac{\delta}{64L}   \sum_{t=0}^T \sum_{i=0}^{t-1} w_{i} \left(1-\frac{\delta}{8}\right)^{t-i}  \E \norm{\nabla f(\xx_i)}^2  + \sigma^2 \sum_{t=0}^T w_t \gamma_t \\
 &\leq \frac{\delta}{64L}\sum_{t=0}^T w_t  \E \norm{\nabla f(\xx_t)}^2 \sum_{i=0}^{\infty} \left(1-\frac{\delta}{8}\right)^{i} + \sigma^2 \sum_{t=0}^T w_t \gamma_t\,.
\end{align*}
Observing $\sum_{i=0}^\infty (1-\delta/8)^i \leq \frac{8}{\delta}$ concludes the proof.
\end{proof}

\noindent With this lemma we can now complete the proof analogously to the proof of Theorem~\ref{thm:delay}, as we outline next.
\begin{proof}[Proof of Theorem~\ref{thm:sparsified}]
We observe that by setting $\tau := \frac{2}{\delta}$ in Theorem~\ref{thm:sparsified} and Lemma~\ref{lemma:sparse_final}, respectively, we fall back in the setting of Theorem~\ref{thm:delay} and Lemma~\ref{lemma:dsgd_final}. Thus the proof follows along the same lines and we will not repeat it here. There is only one small caveat: Lemma~\ref{lemma:sparse_final} requires $2\tau$-slow increasing weights, instead of $\tau$-slow increasing weights as in Lemma~\ref{lemma:dsgd_final}. However, it can easily be checked that this condition is satisfied (see also the remarks in the proof of Theorem~\ref{thm:delay}).
\end{proof}

\section{Local SGD with Infrequent Communication}
\label{sec:localsgd}
In this section we analyze local SGD (parallel SGD) with the developed tools. We follow closely~\cite{Stich2018:local} and provide an analysis without the bounded gradient assumption.

The local SGD algorithm evolves $K \geq 1$ sequences  $\bigl\{\xx_t^k\bigr\}_{t \geq 0}^{k \in [K]}$ in parallel, for an integer $K$. The sequences are synchronized every $\tau \geq 1$ iterations, in the following way:
\begin{align}
 \xx_{t+1}^k = \begin{cases} \frac{1}{K}\sum_{k=1}^K \bigl(\xx_t^k - \gamma_t \gg_{t}^k\bigr) &\text{if } \tau  \vert (t+1), \\
 \xx_{t}^k - \gamma_t \gg_t^k &\text{otherwise}. \end{cases} \label{def:local}
\end{align}
Extending our notion in a natural way, we denote by $\gg_t^k$ the gradient oracle on worker $k$ at iteration $t$, with $\gg_t^k = \nabla f(\xx_k^t)+ \bxi_t^k$ and we assume Assumption~\ref{ass:noise} for the noise $\bxi_t^k$. 

\noindent Using a similar proof technique as in the previous sections, we can derive the following complexity estimates.
\begin{theorem}
\label{thm:local}
Let $\bigl\{\xx_t^k\bigr\}_{t \geq 0}^{k \in [K]}$ denote the iterates of local SGD~\eqref{def:local} 
with constant stepsize $\{\gamma_t = \gamma\}_{t \geq 0}$
on a differentiable function $f \colon \R^d \to \R$ under assumptions Assumptions~\ref{ass:lsmooth} and~\ref{ass:noise}. Then, if $f$ 
\begin{itemize}[nosep,leftmargin=12pt,itemsep=2pt]
 \item satisfies Assumption~\ref{ass:strong} for $\mu > 0$, then there exists a stepsize $\gamma \leq \frac{1}{10L(\tau K + M)}$ (chosen as in Lemma~\ref{lemma:constant})
such that
 \begin{align*}
  \E f(\xx^{\rm out})-f^\star = \tilde \cO \left( L(\tau K + M) \norm{\xx_0-\xx^\star}^2 \exp \left[- \frac{\mu T}{10L(\tau K +M)} \right] + \frac{\sigma^2}{\mu K T} \right)\,,
 \end{align*}
 where the output $\xx^{\rm out} \in \{\xx_t^k\}_{t-1 \in [T]}^{k \in [K]}$ is chosen to be $\xx_t^k$ with probability proportional to 
 $(1-\mu \gamma/2)^{-t}$ (uniformly over $k \in [K]$).
 \item satisfies Assumption~\ref{ass:strong} for $\mu = 0$, then there exists a stepsize $\gamma \leq \frac{1}{10L(\tau K + M)}$ (chosen as in Lemma~\ref{lemma:general}) such that
 \begin{align*}
 \E f(\xx^{\rm out})-f^\star = \cO \left(  \frac{L(\tau K + M) \norm{\xx_0-\xx^\star}^2 }{ T } + \frac{\sigma \norm{\xx_0-\xx^\star} }{\sqrt {K T}}  \right)\,,
 \end{align*}
 where the output $\xx^{\rm out} \in \{\xx_t^k\}_{t-1 \in [T]}^{k \in [K]}$ is chosen uniformly at random from the iterates $\{\xx_t^k\}_{t-1 \in [T]}^{k \in [K]}$.
 \item is an arbitrary non-convex function, then there exists a stepsize $\gamma \leq \frac{1}{10L(\tau K + M)}$ (chosen as in Lemma~\ref{lemma:general}), such that
  \begin{align*}
   \E \norm{\nabla f(\xx^{\rm out})}^2 = \cO \left(  \frac{L(\tau K + M) (f(\xx_0) - f^\star) }{ T } + \sigma \sqrt{\frac{L(f(\xx_0) - f^\star)}{K T}}  \right)\,.
 \end{align*}
 where the output $\xx^{\rm out} \in \{\xx_t^k\}_{t-1 \in [T]}^{k \in [K]}$ is chosen uniformly at random from the iterates $\{\xx_t^k\}_{t-1 \in [T]}^{k \in [K]}$.
\end{itemize}
\end{theorem}

\citet{Stich2018:local} shows that under the bounded gradient assumption, local SGD can converge on strongly convex functions at the optimal statistical rate $\cO \bigl( \frac{\sigma^2}{\mu KT}\bigr)$ given $T = \Omega(\frac{L\tau^2K}{\mu})$. Here we study a more general setting and prove an iteration complexity of $\tilde \cO \bigl(\frac{\sigma^2}{\mu K\epsilon} + \frac{L(\tau K + M)}{\mu} \bigr)$, i.e.\ the same $\epsilon$ dependency for the statistical term, but much milder dependency on the the optimization term. Local SGD achieves optimal $\cO \bigl( \frac{\sigma^2}{\mu K \epsilon}\bigr)$ iteration complexity  
if $T =\tilde \Omega\bigl(\frac{L\tau K}{\mu} \bigr)$. This improves the previously known bound of $T = \Omega(\tau^2K)$ with quadratic dependence on $\tau$  to the linear $\tilde\Omega(\tau K)$ dependence. 
We now compare these estimates to the complexity of mini-batch SGD. To make the comparison and discussion of results easier to follow, we will express all complexity bounds in this paragraph in terms of total stochastic gradient computations, i.e.\ Theorem~\ref{thm:local} gives for local SGD a complexity estimate of $\tilde \cO \bigl(\frac{\sigma^2}{\mu \epsilon} + \frac{K^2 L\tau}{\mu}\bigr)$ when $M=0$ and $\mu > 0$.
Two settings are of particular interest: (i) first, we consider mini-batch SGD with batch size $K$, that has oracle complexity $\tilde \cO \bigl(\frac{\sigma^2}{\mu\epsilon} + \frac{KL}{\mu}\bigr)$ (recalling the results from Section~\ref{sec:minibatch}). We observe that local SGD reaches the same statistical term with $\tau$-times less communication, but a worse optimization term.
Another interesting setting is the comparison to (ii) mini-batch SGD with much larger batch size $\tau K$ but with the same number of communication rounds. This algorithm has oracle complexity  $\tilde \cO \bigl(\frac{\sigma^2}{\mu \epsilon} + \frac{KL\tau}{\mu}\bigr)$. When $K=\cO(1)$, local SGD has the same complexity as mini-batch SGD under this setting. 

Besides these positive observations, the bound provided here does not seem to be optimal (especially the dependency on $K$). For instance, we see that the estimate becomes vacuous when $\tau = T$. However, we know that we should at least expect (oracle) complexity $\tilde \cO \bigl(\frac{\sigma^2}{\mu \epsilon} + \frac{KL}{\mu}\bigr)$ in this case (convergence of each individual sequence). 
This indicates that our balancing of the statistical and the optimization term is not optimal. In fact, currently the best known lower bound for the oracle complexity of local SGD is $\tilde\Omega\bigl(\frac{\sigma^2}{\mu \epsilon} + K\sqrt{\frac{L}{\mu}}\bigr)$ by \citet{Woodworth2018:graph}. Even ignoring the dependence on the condition number (which is improvable via acceleration), our rates are off by $\cO(\tau K)$.

\citet{Dieuleveut2019:local} present a very detailed analysis of local SGD on strongly convex and smooth functions, not only considering the convergence in function value as we do here, but by providing more refined analysis on the behavior of the iterates, following~\cite{Moulines2011:nonasymptotic,
Dieuleveut2017:harder}. However, they consider only polynomial (Polyak-Ruppert) averaging and do thus not recover the exponential decaying dependency on the initial bias in the complexity estimates. Combining their estimates with the exponential averaging might be an interesting future direction.

\subsection{Proof of Theorem~\ref{thm:local}}
Analogously to the virtual iterate $\tilde \xx_t$ in the previous proofs, we define here a (virtual) averaged iterate $\tilde \xx_t$, by setting $\tilde \xx_0 = \xx_0$ and
\begin{align}
 \tilde \xx_{t+1} := \tilde \xx_t -  \frac{\gamma_t}{K}\sum_{k=1}^K  \gg_t^{k}\,, \qquad \forall t \geq 0. \label{def:bar} %
\end{align}
\noindent 
We need a slightly adapted version of Lemmas~\ref{lemma:main} and~\ref{lemma:main-nonconvex}. 
\begin{lemma}
\label{lemma:main_local}
Let $\bigl\{\xx_t^k\bigr\}_{t \geq 0}^{k \in [K]}$ be defined as in~\eqref{def:local} with gradient oracles $\bigl\{\gg_t^k\bigr\}_{t \geq 0}^{k \in [K]}$ and objective function $f \colon \R^d  \to \R$ as in Assumptions~\ref{ass:strong}--\ref{ass:noise}. If $\gamma_t \leq \frac{K}{4L(K+M)}$, $\forall t \geq 0$, then for $\{\tilde \xx_t\}_{t \geq 0}$ defined as in~\eqref{def:bar},
\begin{align*}
 \E{ \norm{\tilde \xx_{t+1} - \xx^\star}^2 } &\leq
 \left(1-\frac{\mu \gamma_t}{2}\right) \E{\norm{\tilde \xx_{t} - \xx^\star}^2} 
 - \frac{\gamma_t}{2K}\sum_{k=1}^K  \E{(f(\xx_t^k)- f^\star)} + \frac{\gamma_t^2 \sigma^2}{K}
 + \frac{3 L \gamma_t}{K} \sum_{k=1}^K   \E{\norm{\xx_t^k - \tilde \xx_t}^2}\,.
\end{align*}
\end{lemma}
\begin{lemma}
\label{lemma:main_local-nonconvex}
Let $\bigl\{\xx_t^k\bigr\}_{t \geq 0}^{k \in [K]}$ be defined as in~\eqref{def:local} with gradient oracles $\bigl\{\gg_t^k\bigr\}_{k \in [K], t \geq 0}$ and a smooth possibly non-convex function $f \colon \R^d  \to \R$ satisfying Assumptions~\ref{ass:lsmooth} and~\ref{ass:noise}. If $\gamma_t \leq \frac{K}{2L(K+M)}$, $\forall t \geq 0$, then for $\{\tilde \xx_t\}_{t \geq 0}$ defined as in~\eqref{def:bar},
\begin{align*}
 \E f(\tilde \xx_{t+1}) &\leq
 \E f(\tilde \xx_t)
 - \frac{\gamma_t}{4K}\sum_{k=1}^K  \E \norm{\nabla f(\xx_t^k)}^2 + \frac{\gamma_t^2 L \sigma^2}{2 K}
 + \frac{L^2 \gamma_t}{2K} \sum_{k=1}^K   \E{\norm{\xx_t^k - \tilde \xx_t}^2}\,.
\end{align*}
\end{lemma}
\noindent The proofs lemmas~\ref{lemma:main_local} and~\ref{lemma:main_local-nonconvex} are very similar to those of Lemmas~\ref{lemma:main} and~\ref{lemma:main-nonconvex}. We defer them to the appendix.

Similar as in the previous sections, we will now first derive an upper bound on $\E \norm{\xx_t^k - \tilde \xx_t}^2$ in Lemma~\ref{lemma:local_final} below. The theorem then follows analogously to the proofs of Theorem~\ref{thm:delay} and~\ref{thm:sparsified} and we omit it here.
\begin{lemma} \label{lemma:local_final}
Let $\{\tilde \xx_t\}_{t \geq 0}$, $\{\xx_t^k\}_{t \geq 0}^{k \in [K]}$ be defined as above and stepsizes $\{\gamma_t\}_{t \geq 0}$ with $\gamma_{t} \leq \frac{1}{10 L(\tau K +M)}$, $\forall t \geq 0$ and $\{\gamma_t^2\}_{t \geq 0}$ is $\tau$-slow decaying. Then
\begin{align}
 \frac{1}{K} \sum_{k=1}^K \Eb {3L \norm{\xx_t^k - \tilde \xx_t}}^2 \leq \frac{1}{10 L \tau K} \sum_{k=1}^K \sum_{i=0}^{\tau-1} \E\norm{\nabla f(\xx_{t-i}^k)}^2  + 2\frac{\gamma_t \sigma^2}{K} \,. \label{eq:local_bound1}
\end{align}
Furthermore, for any $\tau$-slow increasing non-negative sequence $\{w_t\}_{t \geq 0}$ it holds:
\begin{align*}
 \frac{1}{K}\sum_{k=1}^K \sum_{t=0}^T \Eb {3L \norm{\xx_t^k - \tilde \xx_t}}^2 \leq \frac{ 1 }{5LK} \sum_{k=1}^K \sum_{t=0}^T w_t \E\norm{\nabla f(\xx_{t-i}^k)}^2 + 2 \frac{ \sigma^2}{K} \sum_{t=0}^T w_t \gamma_t\,.  %
\end{align*}
\end{lemma}
\begin{proof}
We start with the first claim. By definition and Lemma~\ref{lem:independent} from above, we have the bound:
\begin{align*}
 \frac{1}{K} \sum_{k=1}^K \E  \norm{\xx_t^k - \tilde \xx_t}^2 & = \frac{1}{K} \sum_{k=1}^K \E \norm{\xx_t^k - \tilde \xx_{\lfloor t/\tau\rfloor \tau} - (\tilde \xx_t - \tilde \xx_{\lfloor t/\tau\rfloor \tau})}^2 \leq \frac{1}{K} \sum_{k=1}^K  \E \norm{\xx_t^k - \tilde \xx_{\lfloor t/\tau\rfloor \tau}}^2 \\
 &\leq \frac{1}{K}\sum_{k=1}^K \E \norm{\sum_{i=0}^{\tau -1} \gamma_{t-i} \bigl(\nabla f(\xx_{t-i}^k) + \bxi_{t-i}^k\bigr)}^2  \\
 &\stackrel{\eqref{eq:independent}}{\leq} \frac{3\tau}{2K} \sum_{k=1}^K \sum_{i=0}^{\tau-1} \gamma_{t-i}^2 \E \norm{\nabla f(\xx_{t-i}^k)}^2 +  \frac{3}{K} \sum_{k=1}^K \sum_{i=0}^{\tau-1} \gamma_{t-i}^2 \E \norm{\bxi_{t-i}^k}^2 \\
 & \stackrel{\eqref{def:noise-general}}{\leq} \frac{3(\tau + M)}{2K} \sum_{k=1}^K \sum_{i=0}^{\tau-1}  \gamma_{t-i}^2 \E\norm{\nabla f(\xx_{t-i}^k)}^2 + 3\sigma^2 \sum_{i=0}^{\tau-1} \gamma_{t-i}^2\,,
\end{align*}
where we used $\E\norm{X - \E X}^2 \leq \E \norm{X}^2$, for random variable $X$, for the first inequality.
For $i \leq \tau$ we have the upper bound $\gamma_{t-i}^2 \leq \gamma_t^2 \bigl(1+\frac{1}{2\tau}\bigr)^\tau \leq \gamma_t^2 \exp \left[\frac{\tau}{2\tau} \right] \leq 2 \gamma_t^2$, as $1+x \leq e^x$, $\forall x \in \R$. Thus we can simplify:
\begin{align}\label{eq:final-bound}
 \E  \norm{\ee_t}^2 \leq \gamma_t^2 \left( \frac{3(\tau +M)}{K} \sum_{k=1}^K \sum_{i=0}^{\tau-1}  \E\norm{\nabla f(\xx_{t-i}^k)}^2 + 6 \tau \sigma^2 \right)\,.
\end{align}
By observing that the choice $\gamma_t \leq \frac{1}{10 L(\tau K + M)}$ implies  $\bigl(3L \cdot 3(\tau + M) \gamma_t^2\bigr) \leq \frac{1}{10 L(\tau K + M)} \leq \frac{1}{10 L\tau K} \leq \frac{1}{10 L\tau}$ 
and $\bigl(3 L \cdot 6\tau \gamma_t\bigr) \leq \frac{2L\tau}{L(\tau K + M)} \leq \frac{2}{K}$ we show the first claim. 
The second claim follows analogously to the proof in Lemma~\ref{lemma:dsgd_final} from~\eqref{eq:local_bound1}.
This concludes the proof.
\end{proof}
We note that our proof can recover the subsequent result of \cite{Woodworth2019local} by using \eqref{eq:final-bound} directly without the further simplification, and choosing a better step size $\gamma_t$ as in \cite{Woodworth2019local}.
\section{Conclusion}

We leverage the error-feedback framework to analyze the effect of different forms of delayed updated in a unified manner. We prove that the effects of such delays is negligible for SGD in the presence of noise. This finding comes as no surprise, as it agrees with previous results for SGD with delayed updates or with gradient compression~\cite{Chaturapruek2015:noise,Arjevani2018:delayed,
Stich2018:sparsified,Karimireddy2019:error}. We improve on these previous work by providing a tighter non-asymptotic convergence analysis in a more general setting. While our analysis matches with known lower bounds in some settings, in others (such as the local SGD) still leaves a gap. A further limitation of the analysis is that in its current form it is restricted to only unconstrained objectives. Overcoming these limitations, as well as incorporating acceleration, are fruitful avenues for future research. 
Here. we studied three forms of delays in well prescribed theoretical forms. Similar results can be derived for asynchronous methods with atomic updates under more general conditions (i.e.\ variable bounded, instead of fixed delays, block-coordinate updates, etc.), as well as a combination of the three delays studied here (e.g.\ local SGD with compressed communication). These results could be worked out in future work if there is concrete need dictated by practice.

\section*{Acknowledgments}
We thank %
Martin Jaggi for comments and pointing out related notions of quasi-convexity, and
Thijs Vogels for his comments on this manuscript.
We acknowledge funding from SNSF grant 200021\_175796, as well as a Google Focused Research Award.

\newpage

\appendix

\section{Deferred Proofs}
\subsection{Proof of Lemma~\ref{lemma:main_local}}

\begin{proof}[Proof of Lemma~\ref{lemma:main_local}]
We expand:
\begin{align*}
  \norm{\tilde \xx_{t+1} - \xx^\star}^2  &\stackrel{\eqref{def:bar}}{=} \norm{\tilde \xx_{t} - \xx^\star}^2 - \frac{2\gamma_t}{K} \sum_{k=1}^K \lin{\gg_t^k, \xx_t^k-\xx^\star} + \frac{\gamma_t^2}{K^2} \norm{ \sum_{k=1}^K \gg_t^k }^2 + \frac{2\gamma_t}{K} \sum_{k=1}^K \lin{\gg_t^k, \tilde \xx_t - \xx_t^k}\,,
\end{align*}
By using independence,
\begin{align*}
 \EE{\bxi_t^1,\dots,\bxi_t^k} \norm{ \sum_{k=1}^K \gg_t^k }^2 =  \norm{\sum_{k=1}^K \nabla f(\xx_t^k)}^2 + \sum_{k=1}^K \E \norm{\bxi_t^k}^2 \stackrel{\eqref{def:noise}}{\leq} 2L(K+M) \sum_{k=1}^K (f(\xx_t^k)-f^\star) + K\sigma^2 \,.
\end{align*}
Thus we can take expectation above:
\begin{align}\begin{split}
  \Eb{\norm{\tilde \xx_{t+1} - \xx^\star}^2 \mid \tilde \xx_t }
  &\stackrel{(\ref{eq:smoothbound})}{\leq} \norm{\tilde \xx_{t+1} - \xx^\star}^2 - \frac{2\gamma_t}{K} \sum_{k=1}^K \lin{\nabla f(\xx_t^k), \xx_t^k-\xx^\star} \\ &\qquad + \frac{2L(K+M)\gamma_t^2}{K^2} \sum_{k=1}^K (f(\xx_t^k)-f^\star)  \\ &\qquad + \frac{\gamma_t^2 \sigma^2}{K}  + \frac{ 2\gamma_t }{K} \sum_{k=1}^K \lin{\nabla f(\xx_t^k), \tilde \xx_t^k - \xx_t}\,. \label{eq:long1_1}
  \end{split}
\end{align} 
By Assumption~\ref{ass:strong}:
\begin{align*}
  -2\lin{\nabla f(\xx_t^k), \xx_t^k-\xx^\star} \stackrel{\eqref{def:strong}}{\leq} - \mu \norm{\xx_t^k- \xx^\star}^2  - 2(f(\xx_t^k)-f^\star)\,,
\end{align*}
and by $2\lin{\aa,\bb} \leq \alpha \norm{\aa}^2 + \alpha^{-1} \norm{\bb}^2$ for $\alpha > 0$, $\aa,\bb \in \R^d$,
\begin{align*}
2\lin{\nabla f(\xx_t^k), \tilde \xx_t - \xx_t^k} \leq \frac{1}{2L} \norm{\nabla f(\xx_t^k)}^2 + 2L\norm{\xx_t^k -\tilde \xx_t}^2 \stackrel{\eqref{def:lsmooth}}{\leq} f(\xx_t^k)-f^\star + 2L \norm{\xx_t^k -\tilde \xx_t}^2 \,. 
\end{align*}
And by $\norm{\aa + \bb}^2 \leq (1+\beta)\norm{\aa}^2 + (1+\beta^{-1}) \norm{\bb}^2$ for $\beta > 0$ (as a consequence of Jensen's inequality), we further observe
\begin{align*}
 -\norm{\xx_t^k - \xx^\star}^2 \leq - \frac{1}{2} \norm{\tilde \xx_t - \xx^\star}^2 + \norm{\xx_t^k - \tilde \xx_t}^2\,.
\end{align*}
Plugging all these inequalities together into~\eqref{eq:long1_1} yields
\begin{align*}
  \Eb{\norm{\tilde \xx_{t+1} - \xx^\star}^2 \mid \tilde \xx_t }
  &\leq \left(1-\frac{\mu \gamma_t}{2}\right) \norm{\tilde \xx_{t+1} - \xx^\star}^2 - \frac{\gamma_t (K - 2L(K+M)\gamma_t)}{K^2} \sum_{k=1}^K (f(\xx_t^k)- f^\star) \\ &\qquad + \frac{ \gamma_t^2 \sigma^2}{K} + \frac{\gamma_t (2L+\mu)}{K} \sum_{k=1}^K \norm{\xx_t^k - \tilde \xx_t}^2 \,. 
\end{align*}
The claim follows by the choice $\gamma_t \leq \frac{K}{4L(K+M)}$ and $L \geq \mu$.
\end{proof}

\subsection{Proof of Lemma~\ref{lemma:main_local-nonconvex}}

\begin{proof}[Proof of Lemma~\ref{lemma:main_local-nonconvex}]
We begin using the definition of $\tilde \xx_{t+1}$ and the smoothness of $f$
\begin{align*}
    f(\tilde \xx_{t+1}) \stackrel{\eqref{def:bar}}{\leq} f(\tilde \xx_t) - \frac{\gamma_t}{K} \sum_{k=1}^K \lin{\nabla f(\tilde \xx_t), \gg_t^k} + \frac{\gamma_t^2 L}{2K^2}\norm{\sum_{k=1}^K\gg_t^k}^2\,.
\end{align*}
With Assumption~\ref{ass:noise} on the noise and 
using independence
we have
\[
     \EE{\bxi_t^1,\dots,\bxi_t^k} \norm{ \sum_{k=1}^K \gg_t^k }^2 \stackrel{\eqref{def:noise-general}}{=}  \norm{\sum_{k=1}^K \nabla f(\xx_t^k)}^2 + \sum_{k=1}^K \E \norm{\bxi_t^k}^2 \stackrel{\eqref{def:lsmooth},\eqref{def:noise-general}}{\leq} (K+M) \sum_{k=1}^K \norm{\nabla f(\xx_t^k)}^2 + K\sigma^2\,.
\]
Thus we proceed by taking expectation on both sides as follows
\begin{align*}
    \EE{\bxi_t^1,\dots,\bxi_t^k}{f(\tilde \xx_{t+1}) | \xx_t} &\leq f(\tilde \xx_t) - \frac{\gamma_t}{K} \sum_{k=1}^K \lin{\nabla f(\tilde \xx_t), \nabla f(\xx_t^k)} + \frac{\gamma_t^2 L}{2K^2}\EE{\bxi_t^1,\dots,\bxi_t^k} \norm{\sum_{k=1}^K\gg_t^k}^2\\
    &\leq f(\tilde \xx_t) - \frac{\gamma_t}{K} \sum_{k=1}^K \lin{\nabla f(\tilde \xx_t), \nabla f(\xx_t^k)} + \frac{\gamma_t^2 L(K+M)}{2K^2}\sum_{k=1}^K \norm{\nabla f(\xx_t^k)}^2 + \frac{L \gamma_t^2 \sigma^2}{2K}\\
    &= f(\tilde \xx_t)  -\left(\frac{\gamma_t}{K} - \frac{\gamma_t^2 L(K+M)}{2K^2} \right)\sum_{k=1}^K \norm{\nabla f(\xx_t^k)}^2  \\ &\qquad + \frac{\gamma_t}{K} \sum_{k=1}^K \lin{\nabla f(\xx_t^k) - \nabla f(\tilde \xx_t), \nabla f(\xx_t^k)} + \frac{L \gamma_t^2 \sigma^2}{2K}\,.
\end{align*}
Note that Cauchy-Schwarz and Jensen inequalities together give  $\lin{\aa, \bb} \leq \tfrac{\beta}{2}\norm{\aa}^2 + \tfrac{1}{2\beta}\norm{\bb}^2$ for any $\beta > 0$. Using this observation with $\beta = 1$ we can proceed as
\begin{align*}
    \sum_{k=1}^K \lin{\nabla f(\xx_t^k) - \nabla f(\tilde \xx_t), \nabla f(\xx_t^k)} &\leq \sum_{k=1}^K\frac{1}{2}\norm{\nabla f(\xx_t^k) - \nabla f(\tilde \xx_t)}^2 + \sum_{k=1}^K\frac{1}{2}\norm{\nabla f(\xx_t^k)}^2\\
    &\stackrel{\eqref{def:lgradlipschitz}}{\leq}\frac{L^2}{2}\sum_{k=1}^K\norm{\xx_t^k - \tilde \xx_t}^2 + \frac{1}{2}\sum_{k=1}^K\norm{\nabla f(\xx_t^k)}^2\,.
\end{align*}
Plugging this back, we get our result that
\[
\EEb{\bxi_t}{f(\tilde \xx_{t+1}) | \xx_t} \leq f(\tilde \xx_t) - \gamma_t\left(\frac{1}{2K} - \frac{\gamma_t L(K+M)}{2K^2}\right)\sum_{k=1}^K\norm{\nabla f(\xx_t^k)}^2 + \frac{\gamma_t^2 L \sigma^2}{2K}  + \frac{\gamma_t L^2}{2 K}\sum_{k=1}^K\norm{\xx_t^k - \tilde \xx_t}^2\,.
\]
Noting that $\gamma_t \leq \frac{K}{2 L(K+M)}$ implies $\gamma_t\left(\frac{1}{2K} - \frac{\gamma_t L(K+M)}{2K^2}\right) \leq \frac{\gamma_t}{4K}$ yields the lemma.
\end{proof}

\vskip 0.2in
\bibliography{delays}

\end{document}